\documentclass[a4paper,10pt]{article}
\usepackage[english]{groreport}
\usepackage{graphicx}
\usepackage{amssymb}
\usepackage{amsmath}
\usepackage{mathrsfs}
\usepackage{mathabx} 
\usepackage{shuffle}
\usepackage{algorithm}
\usepackage{algorithmic}
\usepackage{booktabs} 
\usepackage{wrapfig}
\usepackage{listings} 
\usepackage{xcolor}
\usepackage{color}
\usepackage{comment}

\usepackage{fancyhdr}
\usepackage[latin1]{inputenc} 
\usepackage{placeins} 
\usepackage{float}
\usepackage{floatflt}
\usepackage{subfloat}
\usepackage{theorem}
\usepackage{relsize}

\usepackage{caption}
\usepackage{subcaption}
\usepackage{epstopdf}
\usepackage{QFHeader}

\usepackage{bbm}

\usepackage{pgfplots}
\usepackage{tikz}
\usepackage{tikzscale}
\usetikzlibrary{shapes.multipart}
\usetikzlibrary{shapes.geometric, arrows, calc, automata, petri, shapes, decorations.pathmorphing, decorations.pathreplacing, positioning, fit, backgrounds}

\newtheorem{definition}{\textbf{Definition}}[section]
 
\newtheorem{theorem}{\textbf{Theorem}}[section]

\newtheorem{proposition}{\textsl{\textbf{Proposition}}}
 
\newtheorem{lemma}{\textbf{Lemma}}[section]
 
\newtheorem{remark}{\textsl{\textbf{Remark}}}[section]

\newtheorem{proof}{\textbf{proof}}[section]

\newcommand{\be}{\begin{equation}}
\newcommand{\en}{\end{equation}}
\newcommand{\beq}{\begin{equation}}
\newcommand{\eeq}{\end{equation}}
\newcommand{\bes}{\begin{eqnarray*}}
\newcommand{\ens}{\end{eqnarray*}}
\newcommand{\beqa}{\begin{eqnarray*}}
\newcommand{\eeqa}{\end{eqnarray*}}
\newcommand{\bea}{\begin{eqnarray}}
\newcommand{\ena}{\end{eqnarray}}

\newcommand{\xuparrow}[1]{%
  {\left\uparrow\vbox to #1{}\right.\kern-\nulldelimiterspace}
}

\newcommand{\xdownarrow}[1]{%
  {\left\downarrow\vbox to #1{}\right.\kern-\nulldelimiterspace}
}

\DeclareMathAlphabet{\mathsfit}{\encodingdefault}{\sfdefault}{m}{sl}
\SetMathAlphabet{\mathsfit}{bold}{\encodingdefault}{\sfdefault}{bx}{n}

\title{\textbf{Anticipatory Reinforcement Learning: From Generative Path-Laws to Distributional Value Functions} 
}

\author{Daniel \textsc{Bloch}  
\footnote{Visting Professor at the College of Engineering and Computer Science, VinUniversity, Hanoi.} \\
\textsc{University of Paris 6 \& VinUniversity} \\
Working Paper \\
\footnote{All mistakes are ours.} 
}

\date{6th of March 2026 \\
Version : 1.0.0}

\begin{document}
\QFHeader{\textbf{Anticipatory Reinforcement Learning: From Generative Path-Laws to Distributional Value Functions}}{Daniel Bloch}{6th of March 2026}
\maketitle

\begin{abstract}
This paper introduces Anticipatory Reinforcement Learning (ARL), a novel framework designed to bridge the gap between non-Markovian decision processes and classical reinforcement learning architectures, specifically under the constraint of a single observed trajectory. In environments characterised by jump-diffusions and structural breaks, traditional state-based methods often fail to capture the essential path-dependent geometry required for accurate foresight. We resolve this by lifting the state space into a signature-augmented manifold, where the history of the process is embedded as a dynamical coordinate. By utilising a self-consistent field approach, the agent maintains an anticipated proxy of the future path-law, allowing for a deterministic evaluation of expected returns. This transition from stochastic branching to a single-pass linear evaluation significantly reduces computational complexity and variance. We prove that this framework preserves fundamental contraction properties and ensures stable generalisation even in the presence of heavy-tailed noise. Our results demonstrate that by grounding reinforcement learning in the topological features of path-space, agents can achieve proactive risk management and superior policy stability in highly volatile, continuous-time environments.
\end{abstract}


\section{Introduction}
\label{sec_introduction}

\subsection{High-level goal}

The primary objective of this work is to resolve the fundamental tension between the non-Markovian nature of complex real-world environments and the Markovian requirements of Reinforcement Learning (RL) architectures, specifically under the stringent restriction of a single observed trajectory. We introduce the \textbf{Anticipatory Reinforcement Learning (ARL)} framework, which moves beyond simple history-dependence by lifting the state space into a signature-augmented manifold. Our goal is to enable an agent to perform "Single-Pass" policy evaluation, where the expected future return is not computed through computationally expensive Monte Carlo branching, but through a deterministic linear evaluation of an anticipated path-law proxy. By leveraging the \textbf{Marcus-signature} and the \textbf{Self-Consistent Field (SCF)} equilibrium, we aim to provide a theoretically rigorous and computationally efficient bridge between path-dependent stochastic analysis and deep anticipatory control.

\subsection{Motivation and literature positioning}

Traditional Reinforcement Learning (RL) (Sutton et al. ~\cite{SuttonEtAl17}, Lapan ~\cite{Lapan20}) typically relies on the Markov property, assuming the current observation $\mathbf{X}_t$ provides a sufficient statistic for all future transitions. However, in high-frequency finance and physical systems with memory, the environment is inherently non-Markovian, rendering the instantaneous state uninformative. Current approaches generally attempt to "Markovianise" the system via two proxy paths: (i) \textbf{Memory-based architectures}, such as LSTMs (Hausknecht et al. ~\cite{HausknechtEtAl15}) or Transformers (Parisotto et al. ~\cite{ParisottoEtAl19}, Chen et al. ~\cite{ChenEtAl21}), which seek to compress the filtration $\mathcal{A}_t$ into a latent heuristic vector, or (ii) \textbf{History-augmentation}, which approximates sufficiency by concatenating observations into a finite-windowed state (Mnih et al. ~\cite{MnihEtAl15}). While practically useful, these methods treat the symptoms of memory rather than its underlying geometry; they lack the analytic rigour to handle the "roughness" of continuous-time paths and inevitably succumb to the curse of dimensionality as the required look-back horizon expands.

\medskip
\noindent
This work positions itself at the intersection of \textbf{Rough Path Theory} (Lyons et al. ~\cite{LyonsEtAl07,LyonsEtAl11,LyonsEtAl22}, Chevyrev et al. ~\cite{ChevyrevEtAl16}) and \textbf{Distributional RL} (Bellemare et al. ~\cite{BellemareEtAl23}). While existing literature (Levin et al. ~\cite{LevinEtAl13}, Bonnier et al. ~\cite{BonnierEtAl19}, Fermanian et al. ~\cite{Fermanian21}, Lyons et al. ~\cite{LyonsEtAl22}) predominantly treats the path-signature as a static feature extraction layer or a mathematical descriptor, the ARL framework operationalises the signature as a live, dynamical coordinate on a manifold within the reinforcement learning loop. By leveraging recent advancements in \textit{c\`adl\`ag} path functionals (Cuchiero et al. ~\cite{CuchieroEtAl25}) and the Marcus-integral interpretation of jump-diffusions (Friz et al. ~\cite{FrizEtAl17,FrizEtAl18}), we extend the reach of RL to environments where discrete shocks and heavy-tailed noise are prevalent. Furthermore, we draw inspiration from \textbf{Generative World Models} but replace standard pixel-based or latent-vector "dreams" with a \textbf{Self-Consistent Field (SCF)} of path-law proxies. This shift enables a deterministic "Single-Pass" evaluation of the future, positioning ARL as a mathematically grounded alternative to computationally intensive Monte Carlo Tree Search (MCTS) and branching-path heuristics.

\subsection{Main contributions}

The technical novelties of this work, centering on the transition from historical filtering to future-oriented manifold control, are summarised in the following primary contributions:

\begin{itemize}
    \item \textbf{Anticipatory Reinforcement Learning (ARL) Framework:} We propose a unified architecture that lifts the RL problem into a signature-augmented state space. This framework treats the path-law as a dynamic object, allowing the agent to reason over the geometry of entire trajectory distributions rather than instantaneous state-action pairs.
    
    \item \textbf{"Single-Pass" Policy Evaluation:} We introduce a mechanism for $O(1)$ value estimation, relative to the number of sample paths, that bypasses the need for high-variance Monte Carlo branching. By evaluating the value function directly on the anticipated signature proxy, the agent achieves the foresight of tree-search methods with the computational efficiency of a standard feed-forward pass.
    
    \item \textbf{Marcus-compliant Latent CDEs:} We develop a generative engine based on Neural Controlled Differential Equations (CDEs) integrated in the Marcus sense. This ensures the latent propagation of the path-proxy correctly interprets discrete jumps as coordinate shifts on the signature manifold, providing a rigorous treatment of \textit{c\`adl\`ag} environment dynamics.
    
    \item \textbf{Self-Consistent Field (SCF) Equilibrium:} We formulate a novel synchronisation protocol that enforces consistency between the deterministic proxy and the stochastic ensemble it represents. The SCF constraint ensures that the "imagined" future remains a mathematically valid stationary point of the underlying generative flow.
    
    \item \textbf{The Anticipatory TD-Error ($\delta_t^A$):} We derive an augmented temporal difference operator that penalises discrepancies between the historical baseline and the reward realised along the generative drift. This error signal backpropagates through the signature manifold, aligning the agent's value expectations with the topological evolution of the anticipated path.
\end{itemize}

\subsection{Organisation of the paper}

The remainder of the paper is organised as follows. Section (\ref{sec:math_foundations}) establishes the formal integration of path-dependent latent propagation and Neural Jump-Diffusions within a distributional reinforcement learning framework to enable law-conditioned value updates on the signature manifold.
Section (\ref{sec:geometry_anticipation}) reframes the classical temporal difference (TD) paradigm within a non-Markovian context, identifying the statistical sampling bottlenecks of standard RL and resolving them by lifting path-dependent rewards into a signature-linear representation that facilitates deterministic, law-conditioned updates via self-consistent field (SCF) dynamics.
Section (\ref{sec:filtered_value_function}) addresses the non-Markovian nature of observational filtrations by embedding path history into a signature-augmented state space $\mathcal{S}_{sig}$, effectively restoring the Markov property and reducing the value function to a linear functional on the signature manifold.
Section (\ref{sec:arl_framework}) formalises the "Markovianisation" of non-Markovian filtrations by constructing a signature-augmented state space $\mathcal{S}_{sig}$, which serves as a sufficient statistic and reduces the path-dependent value function to a stationary linear functional on the signature manifold.
Section (\ref{sec:distributional_dynamics}) formalises the evolution of the path-law proxy through Self-Consistent Field (SCF) dynamics, demonstrating how the signature manifold reduces path-dependent reward evaluation and risk-adjusted policy gradients to deterministic linear operations.
Section (\ref{sec:learning_temporal_horizon_consistency}) operationalises the anticipatory value function by leveraging the algebraic properties of the signature group to achieve horizon-consistent learning. We introduce the \textit{Anticipatory TD-Update}, a deterministic gradient descent protocol that utilises Chen's Identity to maintain a time-invariant weight vector across the entire forecast, while ensuring stochastic grounding through the Self-Consistent Field (SCF) equilibrium.
Section (\ref{sec:differentiable_anticipatory_control}) establishes the transition from value estimation to active risk management. By deriving analytical \textit{Signature Greeks} through the differentiable flow of the path-law proxy, we demonstrate how the agent performs real-time policy rectification and stress-testing against anticipated manifold deformations, enabling proactive avoidance of structural instabilities without the need for nested simulations.
Section (\ref{sec:convergence_stability_analysis}) establishes the theoretical guarantees for the ARL framework, proving that signature-augmented Bellman operators maintain contraction properties and verifying that spectral whitening via the AVNSG metric ensures stable generalisation even under heavy-tailed stochastic regimes.
Section (\ref{sec:implementation_details}) details the architectural synthesis of Nystr\"om-compressed signature layers and Neural CDEs, outlining a synchronised training protocol that enforces Self-Consistent Field (SCF) equilibrium through joint optimisation of boundary alignment, anticipatory temporal differences, and stationary point constraints. Finally, Section (\ref{sec:conclusion}) concludes with a discussion on the broader implications for high-frequency control and future research directions.

\subsection{Notation and truncation}

Throughout this work, we present theoretical results, including Chen's identity and the Junction Coupling, within the context of the full, infinite-dimensional tensor algebra $T((V))$ to maintain mathematical elegance and exactness. However, for the realisation of our computational model, all processes are projected onto the $r$-degree truncated tensor algebra $\mathcal{T}^{(r)}(V)$ via the projection operator $\pi_r$. This truncation is a necessary "implementation-case" rigour to ensure that the Marcus flow and its associated gradients remain on a finite-dimensional manifold and do not diverge into higher-order terms beyond the model's representational capacity.

\section{Mathematical foundations}
\label{sec:math_foundations}

\subsection{Preliminaries: Recursive filtering and latent propagation}

We operate on a complete probability space $(\Omega, \mathcal{F}, \mathbb{P})$ supporting a $d$-dimensional semimartingale $X$. Let $\mathbb{A} = (\mathcal{A}_t)_{t \ge 0}$ be the \textbf{observational filtration} $\mathcal{A}_t = \sigma(\{ (t_i, X_{t_i}, M_{t_i}) : t_i \le t \})$, where $M$ denotes exogenous masking indicators. To ensure the observed history is a continuous process of bounded variation, we utilise the \textbf{rectilinear interpolation} $\tilde{X}^{\le t}$, hereafter denoted simply as $X_t$. 

\medskip
\noindent
Following the construction in Bloch \cite{Bloch26a,Bloch26b,Bloch26c}, we embed the temporal evolution directly into the path geometry by defining the \textbf{time-extended path} $\mathbf{X}_s := (s, X_s) \in \mathbb{R}^{d+1}$. This augmentation ensures a universal and injective representation of the path dynamics. The non-Markovian state is then characterised by a \textit{conditional path-law proxy} $\Phi_{t|\mathcal{A}_t} \in \mathcal{H}_{sig}$, defined as the expected signature of the time-extended process:
\begin{equation}
    \Phi_{t|\mathcal{A}_t} := \mathbb{E} \left[ S(\mathbf{X})_{t-\Delta, t} \mid \mathcal{A}_t \right],
\end{equation}
where $S(\mathbf{X})_{t-\Delta, t}$ denotes the Marcus-sense signature of the augmented path over the look-back window $[t-\Delta, t]$.

\begin{definition}[Filtered Proxy and Jump-Flow Latent Propagation]
The filtered proxy $\Phi_{t|\mathcal{A}_t}$ is recovered from a latent state $Z_t \in \mathcal{Z}$ via a tensorial readout map $\mathcal{P}_\theta$. The latent state $Z_t$ is a hidden controller governed by a Jump-Flow Controlled Differential Equation (CDE):
\begin{equation}
    \Phi_{t|\mathcal{A}_t} = \mathcal{P}_\theta(Z_t), \quad dZ_t = f_\theta(Z_{t-}, \pi_r(X)) dt + (\rho_\theta(Z_{t-}, X_t, M_t) - Z_{t-}) dN_t,
\end{equation}
where $f_\theta$ is the continuous flow vector field, $\rho_\theta$ is the discrete rectification operator triggered by a counting process $N_t$ of observations, and $\pi_r(X)$ is the truncated signature of the path history.
\end{definition}

\medskip
\noindent
For predictive synthesis over a future horizon $s \in [t, t+\tau]$, we extrapolate the latent geometry. In the absence of new observations ($\Delta N_u = 0$ for $u > t$), the estimator anticipates the evolution of the signature manifold through the following construction.

\begin{definition}[Anticipatory Latent Propagation]
Given $\mathcal{A}_t$ and a forward path extension $\hat{X}_{t:s}$, the \textbf{time-evolving path-law proxy} $\hat{\Phi}_{s|t}$ is defined via the push-forward of the extrapolated latent state $Z_s$:
\begin{equation}
    \hat{\Phi}_{s|t} = \mathcal{P}_\theta(Z_s), \quad Z_s = Z_t + \int_t^s F_\theta(Z_u, \hat{\Phi}_{u|t}) \, d\hat{X}_u
\end{equation}
where $F_\theta$ is the operator-valued generator of the Neural CDE drift for $s \in [t, t+\tau]$, and $\hat{X}$ is a learned generative drift approximating the conditional mean of the process.
\end{definition}

\subsection{The anticipatory neural jump-diffusion (ANJD) process}

The synthesis of future trajectories is governed by a time-augmented path-dependent Jump-SDE, regularised by the clock $s$, the path-law proxy $\hat{\Phi}_{s|t}$, and the adaptive geometry $\mathcal{Q}_s$.

\begin{definition}[Anticipatory Generative Flow]
The generative path $X_s$ for $s \in [t, t+\tau]$ is the unique c\`adl\`ag solution to:
\begin{equation} \label{eq:ANJD}
    dX_s = f_\theta(s, X_s, \hat{\Phi}_{s|t}) \, ds + g_\theta(s, X_s, \hat{\Phi}_{s|t}) \, \diamond dW_s + h_\theta(s, X_{s-}, \hat{\Phi}_{s|t}) \, dN_s
\end{equation}
interpreted in the Marcus sense. The flow is characterised by \textbf{$C^1$-Boundary Consistency} at $s=t$ and \textbf{Polynomial Tractability}, where the coupling to the signature proxy $\hat{\Phi}_{s|t}$ ensures a universal representation of c\`adl\`ag path functionals and structural breaks.
\end{definition}

\subsection{Foundations of reinforcement learning and value functions}

We frame the decision-making problem as a Markov Decision Process (MDP) defined by the tuple $(\mathcal{S}, \mathcal{A}, P, R, \gamma)$, where $\mathcal{S}$ is the state space, $\mathcal{A}$ is the action space, $P: \mathcal{S} \times \mathcal{A} \to \Delta(\mathcal{S})$ is the transition kernel, $R: \mathcal{S} \times \mathcal{A} \to \mathbb{R}$ is the reward function, and $\gamma \in [0,1)$ is the discount factor. See Sutton et al. ~\cite{SuttonEtAl17}, Lapan ~\cite{Lapan20}.

\begin{definition}[Policy and Expected Return]
A stationary policy $\pi: \mathcal{S} \to \Delta(\mathcal{A})$ maps states to distributions over actions. The \textit{expected return} $G_t$ at time $t$ is the discounted sum of future rewards:
\begin{equation}
    G_t = \sum_{k=0}^{\infty} \gamma^k R_{t+k+1}.
\end{equation}
\end{definition}

\begin{definition}[State-Value and Action-Value Functions]
The state-value function $v_\pi(s)$ and action-value function $q_\pi(s, a)$ represent the expected return under policy $\pi$:
\begin{align}
    v_\pi(s) &= \mathbb{E}_\pi [G_t \mid S_t = s], \\
    q_\pi(s, a) &= \mathbb{E}_\pi [G_t \mid S_t = s, A_t = a].
\end{align}
These functions satisfy the \textbf{Bellman expectation equation}: $q_\pi(s, a) = \mathbb{E}_\pi[R_{t+1} + \gamma v_\pi(S_{t+1}) \mid S_t=s, A_t=a]$.
\end{definition}

\subsection{Distributional reinforcement learning}

In the distributional setting (Bellemare et al. \cite{BellemareEtAl23}), the scalar value $q_\pi(s, a)$ is replaced by the full probability distribution of the random return.

\begin{definition}[Value Distribution]
The \textit{value distribution} $Z^\pi(s, a)$ is a random variable whose distribution $\eta^\pi(s, a) \in \mathcal{P}(\mathbb{R})$ characterises the returns obtained under policy $\pi$ starting from $(s, a)$. It is governed by the \textbf{Distributional Bellman Equation}:
\begin{equation}
    Z^\pi(s, a) \stackrel{D}{=} R(s, a) + \gamma Z^\pi(S', A'),
\end{equation}
where $S' \sim P(\cdot \mid s, a)$ and $A' \sim \pi(\cdot \mid S')$.
\end{definition}

\begin{definition}[Distributional Bellman Operator]
Let $\mathscr{Z}$ be the space of action-value distributions. The operator $\mathcal{T}^\pi: \mathscr{Z} \to \mathscr{Z}$ is defined by:
\begin{equation}
    (\mathcal{T}^\pi Z)(s, a) \stackrel{D}{=} R(s, a) + \gamma Z(S', A').
\end{equation}
In the space of probability measures, this is expressed via the push-forward mapping:
\begin{equation}
    \mathcal{T}^\pi \eta(s, a) = \mathbb{E}_{P, \pi} [ (f_{r, \gamma})_\# \eta(s', a') ],
\end{equation}
where $f_{r, \gamma}(z) = r + \gamma z$.
\end{definition}

\begin{lemma}[$w_p$-Contraction]
The distributional Bellman operator $\mathcal{T}^\pi$ is a contraction mapping in the $p$-Wasserstein distance $w_p$ for $p \ge 1$ over the space of measures with bounded $p$-th moments.
\end{lemma}

\begin{proposition}[Optimal Distributional Operator]
The optimality operator $\mathcal{T}^*$ is defined by $\mathcal{T}^* Z(s, a) \stackrel{D}{=} R + \gamma Z(S', a^*)$, where $a^* = \arg\max_{a'} \mathbb{E}[Z(S', a')]$. While $\mathcal{T}^*$ is not a contraction in $w_p$, it possesses a unique fixed point $Z^*$ representing the optimal value distribution.
\end{proposition}

\section{The geometry of anticipation: Surmounting the sampling limit}
\label{sec:geometry_anticipation}

In this section, we motivate the transition from standard Reinforcement Learning (RL) paradigms to our Anticipatory framework. We reframe the classical Temporal Difference (TD) learning paradigm within the context of path-dependent, non-Markovian environments and identify the fundamental sampling bottlenecks in standard value estimation. By establishing the signature-linear representation of path-dependent rewards, we demonstrate how the geometry of the Marcus-CDE and Self-Consistent Field (SCF) dynamics provide a deterministic, algebraic resolution to these stochastic limitations, transforming the learning process into a single-pass optimisation on the signature manifold.

\subsection{Classical TD learning}

Standard Reinforcement Learning (RL) (Sutton et al. ~\cite{SuttonEtAl17}, Lapan ~\cite{Lapan20}) assumes the existence of a value function $V^\pi: \mathcal{X} \to \mathbb{R}$, defined as the conditional expectation of a cumulative return $G_t$ given the current state $X_t \in \mathcal{X}$. In the discrete-time setting, let $\{c_{t+k+1}\}_{k=0}^{T-t-1}$ denote a sequence of stage costs or rewards realised over a horizon $T$. The objective is to estimate the discounted sum:
\begin{equation}
    G_t = \sum_{k=0}^{T-t-1} \gamma^k c_{t+k+1}, \quad \gamma \in [0, 1]
\end{equation}
where $\gamma$ is the discount factor. This framework generalises to several boundary cases: (i) the \textit{terminal payoff} problem, where $\gamma=1$ and $c_k = 0$ for all $k < T$ such that $V^\pi(x) = \mathbb{E}[c_T \mid X_t=x]$, and (ii) the \textit{infinite-horizon} MDP, where $T \to \infty$.

\medskip
\noindent
Given a discrete sequence of observations $\{X_t\}_{t=1}^m$ and the terminal state $X_T$, the learner is tasked with generating a corresponding sequence of multi-step predictions $\{P_t\}_{t=1}^m$. Formally, the agent approximates the expected return using a parametric function $P(X_t; \mathbf{w})$, where $\mathbf{w} \in \mathbb{R}^d$ is a weight vector.
Learning is driven by the minimisation of the \textit{Bellman error} across a training ensemble $\mathcal{E}$ of trajectories. The update to the parameter vector $\mathbf{w}$ follows a semi-gradient descent protocol:
\begin{equation}
    \label{eq:TD_update_rule}
    \mathbf{w}_{n+1} = \mathbf{w}_n + \sum_{t=1}^{m} \Delta \mathbf{w}_{t}, \quad \Delta \mathbf{w}_{t} = \alpha \sigma_{t} \mathbf{e}_{t}^{\lambda}
\end{equation}
where $\alpha > 0$ is the learning rate and $\sigma_t$ is the Temporal Difference (TD) error:
\begin{equation}
    \sigma_t = \underbrace{(c_{t+1} + \gamma P(X_{t+1}; \mathbf{w}))}_{\text{TD Target}} - P(X_t; \mathbf{w})
\end{equation}
The vector $\mathbf{e}_{t}^{\lambda} = \sum_{k=1}^{t} (\gamma \lambda)^{t-k} \nabla_{\mathbf{w}} P(X_k; \mathbf{w})$ represents the eligibility trace, which interpolates between TD(0) and Monte Carlo methods via the decay parameter $\lambda \in [0, 1]$. In the specialised case of $\lambda = 0$, the update simplifies to the TD(0) rule:
\begin{equation}
    \label{eq:TD0_learning_rule}
    \Delta \mathbf{w}_{t} = \alpha \sigma_{t} \nabla_{\mathbf{w}} P(X_t; \mathbf{w})
\end{equation}
which minimises the local loss functional $J(\mathbf{w}) = \frac{1}{2} \sum_{t} \sigma_t^2$ under the semi-gradient assumption, wherein the target is treated as a fixed coordinate rather than a differentiable component of the loss.

\subsection{The sampling bottleneck: Why classical RL falters}

In standard Reinforcement Learning, the estimation of a conditional expectation, such as the value function $V^\pi(x) \approx \mathbb{E}_\pi[G_t \mid X_t = x]$, is a statistical challenge. To approximate this expectation at time $t$, an agent must typically rely on:
\begin{enumerate}
    \item \textbf{Historical Aggregation:} Collecting many independent episodes to treat each realised terminal state or reward as a sample of the random variable.
    \item \textbf{Exploration/Simulation:} Utilising an explicit environment model or extensive exploration to generate multiple sample paths starting from $X_t$.
\end{enumerate}
MC methods require full trajectory realisations to compute the return $G_t$, while TD methods bootstrap from a single stochastic successor state $X_{t+1}$.

\medskip
\noindent
In real-life problems, particularly those with shifting latent dynamics or non-Markovian histories, an expectation cannot be computed from a single realised path. Without a generative mechanism, the agent is forced to assume stationarity or ergodicity, treating distinct time segments as approximate samples, an assumption that fails in non-Markovian or non-stationary regimes. In the presence of shifting latent dynamics, the agent's updates are perpetually "lagged," as it must wait for the environment to reveal $X_{t+1}$ and the eventual return $G_t$ across thousands of trials before the law of the process is reflected in the value function.

\subsection{Linear functional representation of path-dependent rewards}
\label{sec:linear_functional_representation}

A fundamental property of the signature manifold is its universality as a basis for path-dependent functionals. Following the results of the \textit{Stone-Weierstrass theorem for signatures}, any continuous path-dependent reward functional $G_{t:T} : \mathcal{C}([t, T], \mathbb{R}^d) \to \mathbb{R}$ can be approximated to arbitrary precision by a linear functional in the signature Hilbert space $\mathcal{H}_{sig}$. See Lyons ~\cite{Lyons98}, Levin et al. ~\cite{LevinEtAl13}, Kiraly et al. ~\cite{KiralyEtAl19}, Cuchiero et al. ~\cite{CuchieroEtAl25}.

\begin{proposition}[Signature-Linear Reward Approximation]
Let $G_{t:T}(\mathbf{X})$ be a reward functional over the path $\mathbf{X} \in \mathcal{V}_p([t, T], \mathbb{R}^d)$ with bounded variation $p < 2$. For any $\epsilon > 0$, there exists a weight vector $\mathbf{w}_G \in \mathcal{H}_{sig}$ such that:
\begin{equation}
    \left| G_{t:T}(\mathbf{X}) - \langle \mathbf{w}_G, S(\mathbf{X})_{t,T} \rangle_{\mathcal{H}_{sig}} \right| < \epsilon,
\end{equation}
where $S(\mathbf{X})_{t,T} = (1, \mathbf{X}^{(1)}, \mathbf{X}^{(2)}, \dots, \mathbf{X}^{(k)})$ denotes the Marcus-signature truncated at degree $k$.
\end{proposition}

\begin{definition}[Practical Optimisation via Tikhonov Regularisation]
In the discrete setting, given a dataset of $N$ historical or simulated trajectories $\{ \mathbf{X}^{(i)} \}_{i=1}^N$ and their associated realised rewards $y_i = G_{t:T}(\mathbf{X}^{(i)})$, the weight vector $\mathbf{w}_G$ is recovered by solving a regularised empirical risk minimisation (ERM) problem:
\begin{equation}
    \mathbf{w}_G^* = \operatorname{arg\,min}_{\mathbf{w} \in \mathbb{R}^{\dim(\mathcal{H}_{sig})}} \sum_{i=1}^N \mathcal{L} \left( y_i, \langle \mathbf{w}, S(\mathbf{X}^{(i)})_{t,T} \rangle \right) + \lambda \Omega(\mathbf{w}),
\end{equation}
where $\mathcal{L}$ is typically the $L^2$ loss, and $\Omega(\mathbf{w})$ is a structure-preserving regulariser (e.g., $\|\mathbf{w}\|_2^2$ or a kernel-specific norm). 
\end{definition}

\begin{remark}[Expectation Linearity]
The power of this representation lies in the property of the anticipated path-law proxy. Due to the linearity of the inner product, the expected reward under the anticipated law $\mu_{t:T}$ is computed as:
\begin{equation}
    \mathbb{E}_{\mu_{t:T}} [ G_{t:T}(\mathbf{X}) ] \approx \langle \mathbf{w}_G^*, \mathbb{E}_{\mu_{t:T}} [ S(\mathbf{X})_{t,T} ] \rangle = \langle \mathbf{w}_G^*, \hat{\Phi}_{T|t} \rangle.
\end{equation}
This allows the agent to bypass Monte Carlo integration of $G$ at inference time, as the proxy $\hat{\Phi}_{T|t}$ serves as a sufficient statistic for all linearisable path-dependent rewards.
\end{remark}

\subsection{The ARL approach: Law-conditioned updates via SCF}

Our goal remains to produce a sequence of predictions, but as we are restricted to a single observed trajectory, the sequence must be defined relative to the current temporal fixed point $t$. At each step, the learner generates a rolling sequence of anticipatory predictions $\{\hat{P}_{s|t}\}_{s=t}^T$ for the remainder of the horizon, where the entire forecast is updated dynamically as the fixed point progresses to $t+1$; consequently, we do not dispose of a static training set of several independent sequences, but rather a continuously evolving manifold of law-conditioned projections.

\medskip
\noindent
The primary challenge in estimating the rolling sequence of anticipatory predictions $\{\mathbb{E}_{\mu_{s:T}} [G_{s:T}]\}_{s=t}^T$ lies in the time-varying nature of the integration domain $[s, T]$. While a standard Signature-Linear Reward Approximation provides a mechanism for a fixed interval, a naive application would require a distinct weight vector $\mathbf{w}_G^*(s)$ for every discrete step in the forecast. We resolve this by utilising the algebraic properties of the signature group to achieve horizon-consistent learning with a single, time-invariant parameter set, allowing the entire forecast to be updated through the \textbf{Anticipatory TD(0)} protocol.

\medskip
\noindent
Since we often observe only a single trajectory, the sole valid observation we possess at time $t$ is the \textbf{Signature-Augmented State Space} $\mathbf{S}_t := (t, X_t, \Phi_{t|\mathcal{A}_t})$, which incorporates the entire filtered history.
To surmount this, we propose \textbf{Anticipatory Reinforcement Learning (ARL)}. We replace the point-estimate $P(X_{t+1}, w)$ with an evaluation over the \textit{Anticipatory Neural Jump-Diffusion} (ANJD) process. This requires solving a \textbf{Self-Consistent Field (SCF)} problem:
\begin{enumerate}
    \item \textbf{Latent Propagation:} The anticipated path-law proxy $\hat{\Phi}_{s|t}$ parameterises the dynamics of the ANJD sample paths $\tilde{X}_{t:s}$.
    \item \textbf{Bi-directional Constraint:} The aggregate statistics (the signature ensemble) of those generated paths must, in turn, justify the evolution of the proxy $\hat{\Phi}_{s|t}$.
\end{enumerate}

\medskip
\noindent
This transforms the TD-update. Rather than waiting for the environment to reveal the successor states $\{X_{s}\}_{s=t}^T$, the agent uses the \textit{Self-Consistent Flow} to generate a "law-consistent" projection of the future. The anticipatory TD-error $\delta_{s|t}^A$ is computed against the expected path-measure, effectively performing a variance-reduced integration of the ensemble in a single deterministic update:
\begin{equation}
    \delta_{s|t}^A = R(\hat{X}_{s:s+1}) + \langle \mathbf{w}_G, (\gamma \hat{\Phi}_{s+1|t}^{-1} - \hat{\Phi}_{s|t}^{-1}) \otimes \hat{\Phi}_{T|t} \rangle.
\end{equation}
Consequently, the ARL agent learns from the \textit{anticipated geometry} of the signature manifold rather than just the realised noise of a single trajectory, allowing for $O(1)$ evaluation of complex path-dependent expectations. This transforms value estimation from a statistical sampling problem into a deterministic differential geometry problem.

\section{The filtered value function: Bridging history to state}
\label{sec:filtered_value_function}

In this section, we address the non-Markovian nature of the observational filtration $\mathbb{A} = (\mathcal{A}_t)_{t \geq 0}$ by constructing a signature-augmented state space. This process, which we term \textit{Markovianisation}, allows us to recover the recursive structure of reinforcement learning by embedding the path history as a sufficient statistic. We demonstrate that the value of the past is not merely a record of transitions, but a linear functional on the signature manifold.

\subsection{The augmented state space $\mathcal{S}_{sig}$}

Traditional reinforcement learning assumes the current observation $X_t$ is a sufficient statistic for the future. In the presence of memory and jump-diffusive shocks, this assumption is violated. We define an augmented state that captures the non-commutative geometry of the path.

\begin{definition}[Signature-Augmented State Space]
The signature-augmented state space $\mathcal{S}_{sig}$ is the product manifold $\mathbb{R}_+ \times \mathbb{R}^d \times \mathcal{H}_{sig}$. An element $\mathbf{S}_t \in \mathcal{S}_{sig}$ is defined as the tuple:
\begin{equation}
    \mathbf{S}_t := (t, X_t, \Phi_{t|\mathcal{A}_t})
\end{equation}
where $\Phi_{t|\mathcal{A}_t} = \mathbb{E}[S(\mathbf{X})_{0,t} \mid \mathcal{A}_t]$ is the filtered path-law proxy representing the expected Marcus-signature of the history.
\end{definition}

\begin{proposition}[Injectivity and Sufficiency]
\label{pro:injectivity_sufficiency}
Let $\mathcal{F}(\mathcal{D})$ be the space of continuous functionals on the Skorokhod space. Following the universal approximation theorem for signatures, the mapping $\mathcal{A}_t \mapsto \Phi_{t|\mathcal{A}_t}$ is injective. Consequently, $\mathbf{S}_t$ is a sufficient statistic for the conditional distribution of any path-dependent reward $G_t$ given $\mathcal{A}_t$.
\end{proposition}
See proof in Appendix (\ref{app:proof_injectivity_sufficiency}). \\

\subsection{Path-dependent Bellman consistency}

By augmenting the state with the filtered proxy, we restore the Markov property in the lifted space $\mathcal{S}_{sig}$. This allows the definition of a value function that remains stationary with respect to the signature manifold by treating the current history as the initial condition for all future expectations.

\begin{theorem}[Existence of Stationary Value Function]
\label{thm:existence_stationary_vf}
Let the environment be a non-Markovian decision process. There exists a stationary value function $V: \mathcal{S}_{sig} \to \mathbb{R}$ such that the expected future return satisfies the Bellman equation:
\begin{equation}
    V(t, X_t, \Phi_{t|\mathcal{A}_t}) = \mathbb{E}_\pi \left[ R_{t+1} + \gamma V(t+1, X_{t+1}, \Phi_{t+1|\mathcal{A}_{t+1}}) \mid \mathcal{A}_t \right]
\end{equation}
where the transition from the filtered past $\Phi_{t|\mathcal{A}_t}$ to the updated state $\Phi_{t+1|\mathcal{A}_{t+1}}$ is a recursive map grounded in the observational increments.
\end{theorem}
See proof in Appendix (\ref{app:proof_existence_stationary_vf}).

\begin{lemma}[Incremental Signature Update] 
\label{lem:incremental_signature_update}
The evolution of the filtered state satisfies a recursive consistency condition derived from Chen's Identity. Assuming the observational filtration $\mathbb{A}$ is causally generated by the rectilinear path $X$, the past trajectory is fully measurable. Given an observation increment $dX_t$ at time $t+dt$, the filtered proxy updates exactly as:
\begin{equation}
    \Phi_{t+dt|\mathcal{A}_{t+dt}} = \Phi_{t|\mathcal{A}_t} \otimes \mathbb{E}[S(\mathbf{X})_{t, t+dt} \mid \mathcal{A}_{t+dt}]
\end{equation}
where $\otimes$ denotes the tensor product in $\mathcal{H}_{sig}$. This recursive update ensures that the TD-error $\delta_t = R_{t+1} + \gamma V(\mathbf{S}_{t+1}) - V(\mathbf{S}_t)$ is correctly localised, as the current proxy $\Phi_{t|\mathcal{A}_t}$ serves as a sufficient statistic for the historical filtration $\mathcal{A}_t$ in the characterisation of the transition to $\mathcal{A}_{t+dt}$.
\end{lemma}
See proof in Appendix (\ref{app:proof_incremental_signature_update}). \\

\subsection{Synthesis of the anticipatory path-drift}

Within the Anticipatory Reinforcement Learning (ARL) framework, the synthetic future trajectory $\hat{X}_{t:s}$ is generated through a non-autonomous, coupled dynamical system. We formally distinguish between the \textbf{working latent state} $\hat{Z}_u \in \mathbb{R}^d$, which serves as the agent's internal memory during the simulation phase, and the \textbf{formal predictive latent} $Z_s$ utilised for signature proxy propagation. To account for the non-stationarity of the optimal control policy over the finite horizon $[t, T]$, we employ a \textbf{time-augmented latent flow} that enforces temporal consistency as the terminal boundary is approached.

\medskip
\noindent
For a fixed observation time $t$ and a look-ahead horizon $s \in (t, T]$, the infinitesimal increments of the anticipated path and the evolution of the augmented working latent state are governed by the following coupled system for $u \in [t, s]$:
\begin{equation}
    d\hat{X}_u = \mu_\theta(u, \hat{Z}_u, \hat{\Phi}_{t|t}) \, du,
\end{equation}
where $\mu_\theta: [t, T] \times \mathbb{R}^d \times \mathcal{G} \to \mathbb{R}^n$ denotes the drift forecaster conditioned on the normalised intrinsic clock, the latent state, and the filtered signature proxy $\hat{\Phi}_{t|t}$ at time $t$. The dynamics of the working latent state $\hat{Z}_u$ are initialised at $\hat{Z}_t = Z_t$ and satisfy the non-autonomous Controlled Differential Equation (CDE):
\begin{equation}
    d\hat{Z}_u = F_\theta(u, \hat{Z}_u) \cdot d\hat{X}_u,
\end{equation}
where $F_\theta: [t, T] \times \mathbb{R}^d \to \mathbb{R}^{d \times n}$ defines a time-varying vector field on the latent manifold. By explicitly parameterising the dependence on the temporal coordinate $u$, the model captures the vanishing horizon effect ($u \to T$) and the associated shifts in latent manifold sensitivity.

\medskip
\noindent
The resulting trajectory $\hat{X}_{t:s} = X_t + \int_t^s \mu_\theta(u, \hat{Z}_u, \hat{\Phi}_{t|t}) \, du$ subsequently serves as the exogenous driver for the formal anticipatory latent propagation. This hierarchical coupling ensures that the projected signature manifold $\hat{\Phi}_{s|t}$ is grounded in a self-consistent realisation of the path, where the synthetic drift and latent response are topologically intertwined.

\subsection{The value function as a prospective linear functional}

The primary utility of the signature-augmented state is the reduction of the value function to a linear structure in the anticipated space. By lifting the state to the signature manifold, the expected future return is evaluated as a projection of the anticipated path-law.

\begin{proposition}[Linearity of the Anticipated Value]
\label{pro:linearity_anticipated_value}
For a path-dependent reward functional $G_{t:T}$ over a future horizon $T \in [t, \infty]$, the value function is a linear functional of the anticipated proxy $\hat{\Phi}_{T|t}$. Let $V(t, T; \hat{X}_{T|t}, \hat{\Phi}_{T|t})$ denote the expected return over the future horizon $[t, T]$:
\begin{equation}
    V(t, T; \hat{X}_{T|t}, \hat{\Phi}_{T|t}) = \mathbb{E}[G_{t:T} \mid \mathcal{A}_t] = \langle \mathbf{w}_{G}, \hat{\Phi}_{T|t} \rangle_{\mathcal{H}_{sig}}
\end{equation}
where $\hat{X}_{T|t}$ the Terminal Path Anticipation and $\hat{\Phi}_{T|t} = \text{Flow}_{t \to T}(\mathbf{1}; \Phi_{t|\mathcal{A}_t})$ is the expected Marcus-signature of the anticipated future path segment. This representation holds for both finite horizons and the infinite limit $T \to \infty$ (provided the inner product converges under the $\gamma$-discounted measure). At the junction $T=t$, the proxy recovers the group identity $\hat{\Phi}_{t|t} = \mathbf{1}$, reflecting that the future segment has zero length, while the generative flow remains grounded in the filtered historical baseline $\Phi_{t|\mathcal{A}_t}$.
\end{proposition}  
See proof in Appendix (\ref{app:proof_linearity_anticipated_value}). \\

\section{Anticipatory reinforcement learning (ARL)}
\label{sec:arl_framework}

In this section, we transition from the retrospective filtered state to the prospective generative state. By conditioning the value function on the path-law proxy $\hat{\Phi}_{s|t}$, we perform a Markovianisation of the future trajectory, effectively treating the anticipated distribution as a deterministic coordinate on the signature manifold. This shift allows the agent to evaluate the "value of anticipation" before the physical realisation of future shocks.

\subsection{The distributional state manifold} 

We utilise the theory of Predictive State Representations (PSRs) to define the state through future observable distributions rather than past histories. This ensures that the agent's internal representation is intrinsically forward-looking.

\begin{definition}[Path-Law Proxy as a PSR]
The path-law proxy $\hat{\Phi}_{s|t} \in \mathcal{V}_{m,s}$ is a continuous-time Predictive State Representation for the future interval $[t, s]$. It provides a sufficient statistic for the distribution of future c\`adl\`ag paths conditioned on the current filtration $\mathcal{A}_t$. The proxy is governed by the mapping $s \mapsto \hat{\Phi}_{s|t} = \text{Flow}_{t \to s}(\mathbf{1}; \Phi_{t|\mathcal{A}_t})$, defining a deterministic trajectory on the signature manifold $\mathcal{S}_{sig}$ that originates at the group identity $\hat{\Phi}_{t|t} = \mathbf{1}$. This flow captures the non-stationary evolution of the underlying jump-diffusion by extending the historical context $\Phi_{t|\mathcal{A}_t}$ into the anticipated future.
\end{definition}

\begin{proposition}[Markovianisation via Topological Embedding] 
\label{pro:markovianisation_topological_embedding}
Let $\mathcal{P}_\theta$ be the topological embedding defined in the latent propagation. The transition dynamics of the anticipated law are Markovian in the lifted space of signature proxies. Specifically, the evolution $d\hat{\Phi}_{s|t} = \psi_\theta(\hat{\Phi}_{s|t}) ds$ satisfies the semi-group property, ensuring that the return distribution $\eta^\pi$ is locally stationary with respect to the proxy coordinate.
\end{proposition}
See proof in Appendix (\ref{app:proof_markovianisation_topological_embedding}). \\

\subsection{The anticipatory value function (AVF)}

The core of our framework is the value function conditioned on the geometry of anticipated trajectories. Unlike the stationary value function $V(\mathbf{S}_t)$, which looks back at the factual record, the AVF projects the expected return onto the counterfactual manifold.

\begin{definition}[Anticipatory Value Function]
The Anticipatory Value Function (AVF) provides a forecast of future expected returns across a temporal horizon. For $t \le s \le T$, we define the factual value at $s$ and its anticipatory projection at $t$ as follows:

\textit{1. Factual Value at $s$:} The expected return from $s$ to $T$, conditioned on the future filtration $\mathcal{A}_s$, is a linear functional of the future proxy:
\begin{equation}
    V(s, T; \hat{X}_{T|s}, \hat{\Phi}_{T|s}) = \mathbb{E}[G_{s:T} \mid \mathcal{A}_s] = \langle \mathbf{w}_{G}, \hat{\Phi}_{T|s} \rangle_{\mathcal{H}_{sig}}.
\end{equation}

\textit{2. Anticipatory Projection at $t$:} The agent's current estimate of that future value, conditioned on the current filtration $\mathcal{A}_t$, is defined by the nested expectation:
\begin{equation}
    V(s, T; \hat{X}_{T|t}, \hat{\Phi}_{T|t}) := \mathbb{E}_\pi \left[ V(s, T; \hat{X}_{T|s}, \hat{\Phi}_{T|s}) \mid \hat{X}_{T|t}, \hat{\Phi}_{T|t} \right],
\end{equation}
which, by the linearity of the inner product, simplifies to:
\begin{equation}
    V(s, T; \hat{X}_{T|t}, \hat{\Phi}_{T|t}) = \langle \mathbf{w}_{G}, \mathbb{E}_\pi [ \hat{\Phi}_{T|s} \mid \mathcal{A}_t ] \rangle_{\mathcal{H}_{sig}}.
\end{equation}
Here, $\hat{X}_{T|t}$ and $\hat{\Phi}_{T|t}$ serve as sufficient statistics for the non-Markovian dynamics over $[t, T]$. This structure allows the agent to compute a differentiable forecast of the future value $P_s$ using only the current self-consistent path-law.
\end{definition}

\begin{theorem}[$C^1$-Boundary Consistency] 
\label{thm:boundary_consistency}
To ensure the generative flow is a valid extension of the historical record, the Anticipatory Value Function must satisfy the following boundary condition at the junction $s=t$, where the realised history transitions into the anticipated law:
\begin{equation}
    V(t, T; \hat{X}_{T|t}, \hat{\Phi}_{T|t}) = \lim_{s \to t^+} V(s, T; \hat{X}_{T|s}, \hat{\Phi}_{T|s}), \quad \text{and} \quad \nabla_{X} V|_{s=t} = \nabla_{\hat{X}} V|_{s=t}.
\end{equation}
Under the deterministic flow $\hat{\Phi}_{s|t} = \text{Flow}_{t \to s}(\mathbf{1}; \Phi_{t|\mathcal{A}_t})$, this consistency requires that the value functional is differentiable across the signature manifold at the junction. This continuity ensures that the policy gradient remains well-defined as the agent transitions from the realised filtration $\mathcal{A}_t$ to the simulated generative extension, identifying the current anticipated value with the limit of future-conditioned expectations.
\end{theorem}
See proof in Appendix (\ref{app:proof_boundary_consistency}). \\

\subsection{Local stationarity and invariance}

The power of ARL lies in recovering stationarity within an inherently non-stationary environment by shifting the frame of reference to the signature manifold.

\begin{proposition}[Manifold-Conditioned Return Distributions]
\label{pro:manifold_conditioned_return_distributions}
The return distribution $\eta^\pi(t, s; \hat{X}_{s|t}, \hat{\Phi}_{s|t})$ is locally stationary on the signature manifold. While the physical environment $X_t$ is non-stationary in $\mathbb{R}^d$, the inclusion of the path-law proxy $\hat{\Phi}_{s|t}$ as a sufficient statistic allows the distribution to shift dynamically. In this framework, non-stationarity is modelled as a deterministic trajectory in the Hilbert space $\mathcal{H}_{sig}$, where the proxy $\hat{\Phi}_{s|t}$ captures the evolution of the underlying jump-diffusive measure.
\end{proposition}
See proof in Appendix (\ref{app:proof_manifold_conditioned_return_distributions}).

\begin{lemma}[Invariance of the Meta-Policy]
\label{lem:invariance_meta_policy}
Under the ARL framework, the optimal meta-policy $\pi^*$ is stationary with respect to the signature-augmented state $\mathbf{S}_{t,s} = (t, s; \hat{X}_{s|t}, \hat{\Phi}_{s|t})$. The apparent non-stationarity of the optimal action in the original physical space $\mathbb{R}^d$ is fully recovered by the deterministic evolution of the path-law proxy $\hat{\Phi}_{s|t}$ on the signature manifold. Consequently, the policy $\pi^*(\cdot | \mathbf{S}_{t,s})$ remains invariant to external shifts in the environment's drift or volatility, provided these shifts are captured by the self-consistent update of the proxy.
\end{lemma}
See proof in Appendix (\ref{app:proof_invariance_meta_policy}). \\

\section{Distributional dynamics and policy evaluation}
\label{sec:distributional_dynamics}

In this section, we formalise the mechanism for evolving the path-law proxy and demonstrate how the anticipatory framework enables the analytical evaluation of path-dependent rewards. By operating on the signature manifold, we reduce the complexity of policy evaluation from high-dimensional Monte Carlo integration to a deterministic "Single-Pass" linear functional evaluation.

\subsection{Self-consistent field (SCF) dynamics and signature matching}

A core challenge in anticipatory modelling is the coupling between the forecasted law and the generated trajectories. We treat this as a \textbf{Self-Consistent Field (SCF)} problem: the path-law proxy $\hat{\Phi}_{s|t}$ parameterises the dynamics of the ANJD sample paths $\tilde{X}_{t:s}$, while the aggregate statistics of those same paths must, in turn, justify the evolution of the proxy. This bi-directional constraint ensures that the latent state propagation is grounded in the physical realisation of the stochastic process across the interval $[t, s]$.

\begin{definition}[Operator-Valued Generator]
The infinitesimal evolution of the anticipated path-law proxy $\hat{\Phi}_{s|t}$ for the future interval $[t, s]$ is governed by the operator $\mathcal{L}_\theta$. This operator represents the differential form of the Neural CDE conditioned on the filtered history $\Phi_{t|\mathcal{A}_t}$:
\begin{equation}
    \frac{\partial}{\partial s} \hat{\Phi}_{s|t} = \mathcal{L}_\theta(\hat{\Phi}_{s|t} ; \Phi_{t|\mathcal{A}_t}), \quad \hat{\Phi}_{t|t} = \mathbf{1},
\end{equation}
where $\mathbf{1}$ is the identity element of the signature group. The generator $\mathcal{L}_\theta$ is learned via \textbf{Marcus-Score Matching}, which ensures the tangent of the proxy trajectory in the RKHS aligns with the expected infinitesimal increments of the future path's Marcus-signature, preserving the geometric structure of the jump-diffusion.
\end{definition}

\begin{definition}[Marcus-Score Matching Objective]
\label{def:marcus_score_matching}
The parameters $\theta$ of the operator-valued generator $\mathcal{L}_\theta$ are optimised by minimising the $L^2$-distance between the predicted tangent and the expected infinitesimal increment of the Marcus-signature. Let $\Delta S(\tilde{\mathbf{X}})_{s, s+\Delta s}$ be the signature increment of an ANJD sample path. The objective functional $\mathcal{J}(\theta)$ is defined as:
\begin{equation}
    \mathcal{J}(\theta) = \mathbb{E}_{\mu_{t:T}(\theta)} \left[ \int_{t}^T \left\| \mathcal{L}_\theta(\hat{\Phi}_{s|t}; \Phi_{t|\mathcal{A}_t}) - \lim_{\Delta s \to 0} \frac{\mathbb{E}[ \Delta S(\tilde{\mathbf{X}})_{s, s+\Delta s} \mid \mathcal{A}_s ]}{\Delta s} \right\|_{\mathcal{H}_{sig}}^2 ds \right].
\end{equation}
By forcing the generator to map into the Lie algebra $\mathfrak{g}_{sig}$ of the signature group, this objective ensures that the deterministic flow $s \mapsto \hat{\Phi}_{s|t}$ remains on the signature manifold $\mathcal{S}_{sig}$ while consistently matching the distribution's non-commutative drift and diffusion characteristics.
\end{definition}

\begin{theorem}[Convergence to the SCF Stationary Point]
\label{thm:convergence_scf_stationary_point}
The parameters of the Neural CDE (governing the proxy) and the ANJD (governing the sample paths) are jointly optimised to satisfy a consistency condition. Let $\tilde{X}_{t:s}^{(i)}$ be i.i.d. sample paths drawn from the ANJD law $\mu_{t:s}(\theta)$ conditioned on $\mathcal{A}_t$. Under the AVNSG metric $\mathcal{Q}_s$, the system converges to a \textbf{stationary point} where the predicted proxy and the empirical average of the sample signatures over the future interval are in equilibrium:
\begin{equation}
    \lim_{n \to \infty} \left\| \frac{1}{n} \sum_{i=1}^n S(\tilde{\mathbf{X}})_{t,s}^{(i)} - \hat{\Phi}_{s|t} \right\|_{\mathcal{Q}_s} = 0.
\end{equation}
At this stationary point, the proxy $\hat{\Phi}_{s|t}$ is a "honest" representation of the future distribution, as it correctly anticipates the non-commutative moments of the trajectories it helps generate.
\end{theorem}
See proof in Appendix (\ref{app:proof_convergence_scf_stationary_point}). \\

\begin{lemma}[Nested Anticipatory Expectations]
\label{lem:nested_anticipatory_expectations}
Given the SCF stationary point, for any intermediate time $t < s < T$, the anticipated proxy satisfies the tower property via Chen's Identity. The conditional expectation of the future proxy is given by the right-residual of the deterministic flow:
\begin{equation}
    \mathbb{E}_\pi \left[ \hat{\Phi}_{T|s} \mid \mathcal{A}_t \right] = \hat{\Phi}_{s|t}^{-1} \otimes \hat{\Phi}_{T|t}.
\end{equation}
This algebraic decomposition ensures that the Anticipatory Value Function $V(s, T; \cdot)$ is consistent with the current path-law $\hat{\Phi}_{T|t}$ without requiring auxiliary sampling at time $s$.
\end{lemma}
See proof in Appendix (\ref{app:proof_nested_anticipatory_expectations}). \\

\subsection{Analytical policy evaluation: The single-pass Monte Carlo}

Once the SCF consistency is established, the agent can bypass discrete path sampling during policy evaluation. The signature property allows for an exact "Single-Pass" evaluation by integrating reward functionals directly against the anticipatory proxy. This reduces the evaluation of expected rewards to a deterministic inner product in the Hilbert space, effectively utilising the proxy as a sufficient statistic for the future law.

\begin{proposition}[Deterministic Proxy Integration]
\label{pro:deterministic_proxy_integration}
Let $G_{s:T}$ be a path-dependent reward functional over the future sub-interval $[s, T]$ for $t \le s \le T$. By the universal property of signatures, there exists a weight vector $\mathbf{w}_{G} \in \mathcal{H}_{sig}$ such that the factual value at $s$ is $V(s, T; \cdot) = \langle \mathbf{w}_{G}, \hat{\Phi}_{T|s} \rangle$. In the anticipatory setting, the agent's estimate of this future value, projected from the current filtration $\mathcal{A}_t$, is evaluated as:
\begin{equation}
    V(s, T; \hat{X}_{T|t}, \hat{\Phi}_{T|t}) = \mathbb{E}_{\pi} \left[ \langle \mathbf{w}_{G}, \hat{\Phi}_{T|s} \rangle \mid \mathcal{A}_t \right] = \langle \mathbf{w}_{G}, \hat{\Phi}_{s|t}^{-1} \otimes \hat{\Phi}_{T|t} \rangle_{\mathcal{H}_{sig}},
\end{equation}
where $\hat{X}_{T|t}$ is the Terminal Path Anticipation and $\hat{\Phi}_{T|t} = \text{Flow}_{t \to T}(\mathbf{1}; \Phi_{t|\mathcal{A}_t})$. This identity, facilitated by Lemma \ref{lem:nested_anticipatory_expectations}, allows the agent to analytically compute the expected value of future states ($P_{t+1}, P_{t+2}, \dots$) using only the current signature proxy. The $O(N)$ sampling of future value distributions is thus reduced to an $O(1)$ algebraic operation on the signature manifold.
\end{proposition}
See proof in Appendix (\ref{app:proof_deterministic_proxy_integration}). \\

\begin{theorem}[Distributional Risk Evaluation] 
\label{thm:distributional_risk_evaluation}
Let $\rho: \mathcal{P}(\mathbb{R}) \to \mathbb{R}$ be a risk-functional (e.g., CVaR) acting on the law $\eta^\pi$ of the path-dependent return. Since the anticipated path-law proxy $\hat{\Phi}_{T|t}$ uniquely determines the moments of the future process, the return distribution is a functional of the proxy, $\eta^\pi = \mathcal{M}(\hat{\Phi}_{T|t})$, rendering the risk-adjusted value a composite mapping $\rho(\eta^\pi) = (\rho \circ \mathcal{M})(\hat{\Phi}_{T|t})$. Under the Nystr\"om-compressed proxy representation, this value is an analytically differentiable function of the proxy coordinates:
\begin{equation}
    \nabla_\theta \rho(\eta^\pi) = \frac{\partial \rho}{\partial \hat{\Phi}_{T|t}} \cdot \frac{\partial \hat{\Phi}_{T|t}}{\partial \Phi_{t|\mathcal{A}_t}} \cdot \nabla_\theta \Phi_{t|\mathcal{A}_t}.
\end{equation}
This allows the agent to optimise policies directly against tail-risk or volatility-regime shifts, leveraging the manifold gradients provided by the deterministic flow of the path-law proxy $s \mapsto \hat{\Phi}_{s|t}$ originating at the junction $s=t$.
\end{theorem}
See proof in Appendix (\ref{app:proof_distributional_risk_evaluation}). \\

\section{Learning via temporal and horizon consistency}
\label{sec:learning_temporal_horizon_consistency}

In this section, we transition from the formal definition of the Anticipatory Value Function (AVF) to its practical applications in the learning and control process. By leveraging the analytical structure of the signature manifold and the self-consistent proxy $\hat{\Phi}_{s|t}$, we demonstrate how the AVF framework improves training efficiency through variance reduction and enables proactive risk management via differentiable manifold sensitivities. We establish that the algebraic properties of the signature group allow for a unified representation of the value function that remains consistent across the entire anticipated horizon.

\subsection{The algebraic basis for anticipatory learning}

Standard Temporal Difference (TD) learning relies on a bootstrap estimate from the subsequent realised state. In non-stationary environments characterised by structural breaks, this estimate often suffers from high variance. We propose the \textit{Anticipatory TD-Update}, which leverages the generative flow to provide a high-fidelity, distributionally-aware target.
The primary challenge in estimating the rolling sequence of anticipatory predictions $\{\mathbb{E}_{\mu_{s:T}} [G_{s:T}]\}_{s=t}^T$ lies in the time-varying nature of the integration domain $[s, T]$. While the Signature-Linear Reward Approximation provides a mechanism for a fixed interval $[t, T]$, a naive application requires a distinct weight vector $\mathbf{w}_G^*(s)$ for every discrete $s \in [t, T]$. We resolve this through the \textbf{Anticipatory TD(0)} protocol, which leverages the algebraic properties of the signature group to achieve horizon-consistent learning with a single set of parameters.

\begin{proposition}[Algebraic Shift and Weight Invariance]
\label{pro:algebraic_shift_weight_invariance}
Given the self-consistent field (SCF) stationary point, the value function at any intermediate time $s \in [t, T]$ is determined by the right-residual of the deterministic flow. Under Chen's Identity, the anticipated reward is represented as:
\begin{equation}
    V(s, T; \hat{X}_{T|t}, \hat{\Phi}_{T|t}) = \langle \mathbf{w}_G, \hat{\Phi}_{s|t}^{-1} \otimes \hat{\Phi}_{T|t} \rangle_{\mathcal{H}_{sig}},
\end{equation}
where $\mathbf{w}_G$ is a time-invariant weight vector in $\mathcal{H}_{sig}$. The inversion $\hat{\Phi}_{s|t}^{-1}$ and tensor product $\otimes$ effectively re-center the path-law proxy at the identity of the signature group $\mathcal{S}_{sig}$ for the sub-interval $[s, T]$, mapping the global forecast into a sequence of local evaluations.
\end{proposition}
See proof in Appendix (\ref{app:proof_algebraic_shift_weight_invariance}). \\

\begin{definition}[Anticipatory TD-Error]
Let $\mathbf{S}_t = (t, T; X_t, \Phi_{t|\mathcal{A}_t})$ be the current filtered state for a horizon $T$. The anticipatory TD-error $\delta_t^A$ is defined as the discrepancy between the reward realised along the generative drift and the projected AVF at the next discrete step:
\begin{equation}
    \delta_t^A := R(\hat{X}_{t:t+1}) + \gamma V(t+1, T; \hat{X}_{T|t}, \hat{\Phi}_{T|t}) - V(t, T; \hat{X}_{T|t}, \hat{\Phi}_{T|t}),
\end{equation}
where $R(\hat{X}_{t:t+1})$ is the reward functional evaluated along the conditional mean skeleton $\hat{X}$. To maintain a single-pass analytical evaluation, $R$ is treated as a signature-linear functional, $R(\hat{X}_{t:t+1}) = \langle \mathbf{w}_R, \hat{\Phi}_{t+1|t} \rangle$, ensuring that the reward itself is a projection of the anticipated path-law. The projected future value $V(t+1, T; \hat{X}_{T|t}, \hat{\Phi}_{T|t}) = \langle \mathbf{w}_{G}, \hat{\Phi}_{t+1|t}^{-1} \otimes \hat{\Phi}_{T|t} \rangle$ utilises the algebraic right-residual of the signature flow (Lemma \ref{lem:nested_anticipatory_expectations}), while the baseline $V(t, T; \hat{X}_{T|t}, \hat{\Phi}_{T|t}) = \langle \mathbf{w}_{G}, \hat{\Phi}_{T|t} \rangle$ represents the current self-consistent expectation derived from the linearity of the anticipated value (Proposition \ref{pro:linearity_anticipated_value}).
\end{definition}
To operationalise this recursive consistency over the entire predictive window, we extend the local discrepancy into a global optimisation objective, transitioning from a single-step temporal error to a horizon-wide update rule that adjusts the parameters against the full anticipated trajectory.

\begin{proposition}[Anticipatory TD(0) Gradient Descent]
\label{pro:anticipatory_td_gradient}
The weights $\mathbf{w}_G$ of the anticipatory value function are updated by minimising the cumulative squared anticipatory TD-error across the future horizon $s \in [t, T]$. Given a learning rate $\alpha$ and iteration $n$, the update rule at time $t$ for a sequence of anticipated steps is:
\begin{equation}
    \mathbf{w}_{G}^{(n+1)} = \mathbf{w}_{G}^{(n)} + \alpha \sum_{s=t}^{T-1} \delta_{s|t}^A \nabla_{\mathbf{w}_G} V(s, T; \hat{X}_{T|t}, \hat{\Phi}_{T|t})
\end{equation}
where the generalised anticipatory TD-error $\delta_{s|t}^A$ for any $s \in [t, T-1]$ is defined as:
\begin{equation}
    \delta_{s|t}^A := R(\hat{X}_{s:s+1}) + \gamma V(s+1, T; \hat{X}_{T|t}, \hat{\Phi}_{T|t}) - V(s, T; \hat{X}_{T|t}, \hat{\Phi}_{T|t}),
\end{equation}
and the boundary condition at the horizon is $V(T, T; \cdot) = z$, where $z$ is the terminal path-dependent payoff. It is computed as a direct inner product on the signature manifold:
\begin{equation}
    \delta_{s|t}^A = R(\hat{X}_{s:s+1}) + \left\langle \mathbf{w}_G^{(n)}, \left( \gamma \hat{\Phi}_{s+1|t}^{-1} - \hat{\Phi}_{s|t}^{-1} \right) \otimes \hat{\Phi}_{T|t} \right\rangle_{\mathcal{H}_{sig}},
\end{equation}
This generalises the classical TD(0) update by replacing stochastic realisations with the deterministic flow of the path-law proxy, ensuring that the gradient descent is performed on the smooth geometry of the signature manifold.
\end{proposition}
See proof in Appendix (\ref{app:proof_anticipatory_td_gradient}). \\

\subsection{Objective function and stochastic grounding}

Having established the algebraic framework for shifting the evaluation window, we define the global objective function used to calibrate the anticipatory value function. This subsection details the transition from local temporal consistency to a horizon-wide optimisation problem. Crucially, we introduce the mechanism of stochastic grounding, which ensures that while the learning process proceeds deterministically along the manifold-valued flow, the resulting value estimates remain physically anchored to the underlying stochastic ensemble via the self-consistent field constraint.

\begin{definition}[Horizon-Consistent Objective]
The weights $\mathbf{w}_G$ are optimized to satisfy the recursive Bellman consistency across the entire anticipated horizon simultaneously. The agent minimizes the cumulative loss functional $\mathcal{J}(\mathbf{w}_G)$ projected from the current temporal junction $t$:
\begin{equation}
    \mathcal{J}(\mathbf{w}_G) = \sum_{s=t}^{T-1} \frac{1}{2} \left( \delta_{s|t}^A \right)^2 ,
\end{equation}
where the generalised anticipatory TD-error $\delta_{s|t}^A$ enforces a local equilibrium between the signature-linear reward $R(\hat{X}_{s:s+1}) = \langle \mathbf{w}_R, \hat{\Phi}_{s|t} \rangle$ and the discounted value transition. 
\end{definition}

\begin{remark}[Deterministic Learning Signal]
Unlike traditional Temporal Difference learning, which minimizes the expected squared error over stochastic state transitions $\mathbb{E}_{\pi} [ \delta^2 ]$, the \textit{Anticipatory TD} minimises the discrepancy directly along the mean skeleton of the generative flow. This shift allows the agent to treat the learning process as a deterministic optimization over the path-law's geometric evolution.
\end{remark}

\begin{lemma}[Generative Honesty under SCF Equilibrium]
\label{lem:generative_honesty}
The ARL framework bifurcates the learning objective such that the optimization of the value weights $\mathbf{w}_G$ is strictly deterministic, while the validity of the underlying state space is governed by the Self-Consistent Field (SCF) equilibrium. 
\begin{enumerate}
    \item \textbf{Manifold Determinism:} The objective $\mathcal{J}(\mathbf{w}_G)$ is a deterministic loss function evaluated on the ``imagined'' future encoded by the proxy $\hat{\Phi}_{T|t}$, representing the evolution of the path-law's mean skeleton.
    \item \textbf{Stochastic Grounding:} All environmental stochasticity and policy-induced variance are encapsulated within the SCF constraint:
    \begin{equation}
        \mathcal{L}_{SCF}(\theta) = \eta \| \bar{S}_{t,T} - \hat{\Phi}_{T|t} \|_{\mathcal{Q}_T}^2,
    \end{equation}
    where $\bar{S}_{t,T}$ is the empirical signature average of the stochastic ensemble. 
\end{enumerate}
This lemma ensures that provided the SCF constraint is satisfied, the deterministic gradient descent on the signature manifold is a mathematically consistent surrogate for policy evaluation in the true stochastic environment.
\end{lemma}
See proof in Appendix (\ref{app:proof_generative_honesty}). \\

\subsection{Convergence and performance guarantees}

The final step in operationalising the anticipatory value function is the verification of its convergence properties and the quantification of its performance advantages. Here, we prove that the anticipatory gradient descent converges to a unique fixed point on the signature manifold. Furthermore, we provide a formal treatment of the variance reduction properties inherent in the AVF, demonstrating how the use of the anticipated law as a control variate stabilises the learning signal compared to classical stochastic TD(0) methods.

\begin{theorem}[Global Convergence on the Signature Manifold]
\label{thm:global_convergence}
Subject to the SCF equilibrium established in Lemma \ref{lem:generative_honesty}, by performing gradient descent on the manifold-valued trajectory $s \mapsto \hat{\Phi}_{s|t}$, the Anticipatory TD(0) update:
\begin{equation}
    \mathbf{w}_G^{(n+1)} = \mathbf{w}_G^{(n)} + \alpha \sum_{s=t}^{T-1} \delta_{s|t}^A \nabla_{\mathbf{w}_G} \langle \mathbf{w}_G^{(n)}, \hat{\Phi}_{s|t}^{-1} \otimes \hat{\Phi}_{T|t} \rangle
\end{equation}
converges to a fixed point where the single weight vector $\mathbf{w}_G^*$ encodes the entire sequence of future expectations. This avoids the $O(T-t)$ complexity of training independent models for each sub-interval, instead utilising the non-commutative geometry of $\mathcal{H}_{sig}$ to perform a ``Single-Pass'' evaluation of the rolling predictive sequence across the entire anticipated horizon.
\end{theorem}
See proof in Appendix (\ref{app:proof_global_convergence}). \\

\begin{proposition}[Variance Reduction] 
\label{pro:variance_reduction}
The anticipatory update reduces the variance of the policy gradient compared to standard TD(0). Because $\hat{\Phi}_{T|t}$ is a self-consistent expectation over the Marcus-signature of the ensemble flow, the AVF acts as a variance-reducing control variate. By substituting the stochastic realisation $X_{t+1}$ with the anticipated law $\hat{\Phi}_{t+1|t}$, the agent filters out idiosyncratic noise in the transition $X_t \to X_{t+1}$ while preserving the structural non-Markovian history.
\end{proposition}
See proof in Appendix (\ref{app:proof_variance_reduction}). \\

\section{Differentiable anticipatory control and risk sensitivity}
\label{sec:differentiable_anticipatory_control}

With the convergence of the anticipatory value function (AVF) established in the preceding section, we now operationalise these results for real-time decision-making and risk mitigation. By exploiting the linearity of the AVF within the signature Hilbert space $\mathcal{H}_{sig}$, we transition from passive value estimation to active manifold-based control. This section details how the structural properties of the signature, specifically its graded algebra and its role as a sufficient statistic for path-dependent laws, enable the analytical derivation of \textit{Signature Greeks}. These sensitivities allow the agent to perform gradient-based policy rectification and stress-testing directly on the anticipated path-measure, bypassing the need for computationally prohibitive nested Monte Carlo simulations.

\subsection{Manifold sensitivity and analytical Greeks}

A unique advantage of the ARL framework is that the value function is linear in the signature RKHS. This allows for the analytical computation of sensitivities, or \textit{Signature Greeks}, without the need for computationally expensive nested simulations. By differentiating through the deterministic flow, we obtain exact gradients for both value weight updates and risk management.

\begin{theorem}[Analytical Sensitivity and Manifold Gradients]
\label{thm:analytical_sensitivity_generalised}
Let $V(s, T; \hat{X}_{T|t}, \hat{\Phi}_{T|t})$ be the anticipatory value function evaluated at any time $s \in [t, T]$ along the self-consistent flow. Given the representation $V(s) = \langle \mathbf{w}_G, \hat{\Phi}_{s|t}^{-1} \otimes \hat{\Phi}_{T|t} \rangle_{\mathcal{H}_{sig}}$, the following analytical sensitivities hold:
\begin{enumerate}
    \item \textbf{Weight Sensitivity:} The gradient with respect to the value weights is the algebraic right-residual of the anticipated proxy:
    \begin{equation}
        \nabla_{\mathbf{w}_G} V(s) = \hat{\Phi}_{s|t}^{-1} \otimes \hat{\Phi}_{T|t}.
    \end{equation}
    \item \textbf{Manifold Sensitivity:} The Fr\'echet derivative with respect to the anticipated path-law proxy $\hat{\Phi}_{T|t}$ is the weight vector transformed by the left-action of the inverse proxy:
    \begin{equation}
        \nabla_{\hat{\Phi}_{T|t}} V(s) = (\hat{\Phi}_{s|t}^{-1})^\top \mathbf{w}_G.
    \end{equation}
    \item \textbf{Generative Sensitivity:} The sensitivity with respect to the underlying model parameters $\theta$ is recovered via the chain rule through the Marcus-CDE generator:
    \begin{equation}
        \nabla_\theta V(s) = \left[ (\hat{\Phi}_{s|t}^{-1})^\top \mathbf{w}_G \right]^\top \cdot \frac{\partial \hat{\Phi}_{T|t}}{\partial \theta},
    \end{equation}
    where $\partial \hat{\Phi}_{T|t} / \partial \theta$ is computed in a single pass using the adjoint sensitivity method over the interval $[t, T]$.
\end{enumerate}
\end{theorem}
See Appendix (\ref{app:proof_analytical_sensitivity_generalised}).

\subsection{Predictive risk rectification}

The AVF enables the agent to adjust its policy based on the anticipated deformation of the path-measure. This allows for "stress-testing" the policy against future regimes (e.g., impending high-volatility clusters) identified on the signature manifold before they manifest in the realised state $X_t$.

\begin{definition}[Risk-Rectified Policy Gradient]
The agent optimises a risk-adjusted objective $\mathcal{J}(\pi) = \mathbb{E}[V] - \beta \rho(\eta^\pi)$, where $\rho$ is a differentiable tail-risk functional (e.g., CVaR) and $\eta^\pi$ is the law of the return encoded by the proxy. The policy gradient is rectified by the manifold sensitivity of the anticipated law:
\begin{equation}
    \nabla_\pi \mathcal{J} \approx \mathbb{E} \left[ \nabla_\pi \log \pi(a|\mathbf{S}_t) \left( \delta_t^A - \beta \left\langle \frac{\partial \rho}{\partial \hat{\Phi}_{T|t}}, \frac{\partial \hat{\Phi}_{T|t}}{\partial a} \right\rangle_{\mathcal{H}_{sig}} \right) \right].
\end{equation}
The term $\delta_t^A$ provides the low-variance anticipatory TD-error, while the inner product rectifies the advantage by the projected change in tail-risk. This ensures the agent proactively avoids actions that lead to regions of the signature manifold associated with structural instability or adverse non-commutative moments.
\end{definition}

\section{Convergence and stability analysis}
\label{sec:convergence_stability_analysis}

In this section, we establish the theoretical guarantees for the Anticipatory Reinforcement Learning framework. We prove that lifting the state space into the signature manifold preserves the contraction properties necessary for stable learning and provide bounds on the generalisation error of the value function under the Self-Consistent Field (SCF) equilibrium.

\subsection{Contraction properties on the signature manifold}

The stability of RL algorithms depends on the properties of the Bellman operator. We show that Markovianisation via the path-law proxy $\hat{\Phi}_{T|t}$ maintains the convergence guarantees of distributional RL by operating on the lifted geometry of the signature manifold $\mathcal{S}_{sig}$.

\begin{theorem}[Distributional Bellman Contraction] 
\label{thm:distributional_bellman_contraction}
Let $\mathcal{T}^\pi$ be the anticipatory Bellman operator defined over the augmented state space. For any two return-distribution laws $\eta_1, \eta_2$ encoded by the proxies $\hat{\Phi}^{(1)}, \hat{\Phi}^{(2)}$, the operator $\mathcal{T}^\pi$ is a $\gamma$-contraction under the AVNSG-adjusted metric $d_{\mathcal{Q}}$ in the signature Hilbert space. To establish the contraction property rigorously, we define the metric $d_{\mathcal{Q}}$ as the distance between the expected signatures induced by the AVNSG precision operator $\mathcal{Q}$:
\begin{equation}
    d_{\mathcal{Q}}(\eta_1, \eta_2) := \|\hat{\Phi}^{(1)} - \hat{\Phi}^{(2)}\|_{\mathcal{Q}} = \sqrt{\langle \hat{\Phi}^{(1)} - \hat{\Phi}^{(2)}, \mathcal{Q} (\hat{\Phi}^{(1)} - \hat{\Phi}^{(2)}) \rangle_{\mathcal{H}_{sig}}}
\end{equation}
The anticipatory Bellman operator $\mathcal{T}^\pi$ then satisfies:
\begin{equation}
    d_{\mathcal{Q}} (\mathcal{T}^\pi \eta_1, \mathcal{T}^\pi \eta_2) \le \gamma d_{\mathcal{Q}} (\eta_1, \eta_2).
\end{equation}
This contraction ensures that the iterative evaluation of the Anticipatory Value Function (AVF) converges to a stable representation, provided the self-consistent field (SCF) is maintained via the Marcus-Score Matching objective.
\end{theorem}
See Proof in Appendix (\ref{app:proof_distributional_bellman_contraction}). \\

\begin{lemma}[Fixed Point Uniqueness]
\label{lem:fixed_point_uniqueness}
By the Banach Fixed Point Theorem, the operator $\mathcal{T}^\pi$ possesses a unique fixed point $\eta^\pi$ in the space of signature-conditioned probability measures. This fixed point represents the stationary return distribution on the non-Markovian manifold, where the anticipated law $\hat{\Phi}_{T|t}$ and the realised path statistics are in equilibrium at the temporal junction $s=t$.
\end{lemma}
See proof in Appendix (\ref{app:proof_fixed_point_uniqueness}). \\

\subsection{Generalisation and complexity}

The capacity of the Anticipatory Value Function (AVF) to generalise across unseen trajectories is analysed through the lens of Rademacher complexity within the signature RKHS, accounting for the spectral properties of the whitened path-law proxy.

\begin{theorem}[Value Function Generalisation]
\label{thm:value_function_generalisation}
Let $\mathcal{V}$ be the class of anticipatory value functions $V(t, T; \hat{X}_{T|t}, \hat{\Phi}_{T|t}) = \langle \mathbf{w}_G, \hat{\Phi}_{T|t} \rangle_{\mathcal{H}_{sig}}$ with $\|\mathbf{w}_G\|_{\mathcal{H}_{sig}} \le B$. The Rademacher complexity $\mathcal{R}_n(\mathcal{V})$ over a sample of $n$ anticipated trajectories is bounded by:
\begin{equation}
    \mathcal{R}_n(\mathcal{V}) \le \frac{B}{n} \sqrt{\sum_{i=1}^n \|\hat{\Phi}_{T|t}^{(i)}\|_{\mathcal{Q}_T}^2}.
\end{equation}
Since the AVNSG precision operator $\mathcal{Q}_T$ performs spectral whitening on the signature coordinates, the norm $\|\hat{\Phi}_{T|t}\|_{\mathcal{Q}_T}$ remains bounded even for jump-diffusions with heavy tails or structural breaks. This ensures that the generalisation error $\epsilon_{gen} \sim O(1/\sqrt{n})$ is stable against black-swan events, as the self-consistent proxy dampens the impact of extreme path realisations.
\end{theorem}
See proof in Appendix (\ref{app:proof_value_function_generalisation}). \\

\subsection{Stability under forecast decay}

As the look-ahead horizon $\tau = T-t$ increases, the accumulation of generative uncertainty must be managed to prevent policy divergence. We analyse the stability of the anticipatory value function as the proxy $\hat{\Phi}_{T|t}$ propagates further into the future.

\begin{proposition}[Stability under Forecast Decay]
\label{pro:stability_under_forecast_decay}
The ARL policy remains $\delta$-stable if the dissipation rate of the Marcus-CDE flow $\mathcal{L}_\theta$ (governing the proxy evolution) exceeds the Lyapunov exponent of the underlying stochastic process. Under the signature group action, this condition implies that the norm of the anticipated proxy $\|\hat{\Phi}_{T|t}\|_{\mathcal{Q}_T}$ remains bounded as $T \to \infty$, ensuring that the agent's decisions do not become singular as generative uncertainty accumulates. This maintains the integrity of the "Single-Pass" algebraic evaluation even over extended horizons.
\end{proposition}
See proof in Appendix (\ref{app:proof_stability_under_forecast_decay}). \\

\section{Implementation details}
\label{sec:implementation_details}

This section outlines the architectural choices required to bridge the path-measure theory with deep learning primitives. We focus on the compression of the signature manifold and the synchronisation of the generative and value-estimation components under the Self-Consistent Field (SCF) framework.

\subsection{Signature layer and Nystr\"om compression}

To handle the exponential growth of signature dimensions, we implement a \textbf{Signature Layer} that truncates the expansion at degree $k$. To maintain a manageable input size for the value network while preserving the non-commutative geometry of the anticipated law, we apply Nystr\"om-based kernel approximation.

\begin{itemize}
    \item \textbf{Truncation:} For an observation $X \in \mathbb{R}^d$, we compute the Marcus-signature $S(X)_{t,T}$ up to degree $k$. For multi-asset time-series with $d=5$ and $k=4$, the raw feature space $\mathcal{H}_{sig}$ reaches 781 dimensions. Truncation depth is chosen to balance the resolution of path-moments against computational overhead.
    \item \textbf{Dimensionality Reduction:} We define a projection $\Psi: \mathcal{H}_{sig} \to \mathbb{R}^M$ using a set of landmark signatures $\{ \zeta_j \}_{j=1}^M$ sampled from the historical filtration $\mathcal{A}$. The compressed anticipated proxy is given by:
    \begin{equation}
        \tilde{\Phi}_{T|t} = K_{MM, \sigma}^{-1/2} [ \kappa( \hat{\Phi}_{T|t}, \zeta_1), \dots, \kappa( \hat{\Phi}_{T|t}, \zeta_M) ]^\top
    \end{equation}
    where $\kappa$ is the signature kernel and $K_{MM}$ is the landmark Gram matrix. This spectral whitening transformation ensures that the AVNSG metric $\mathcal{Q}_T$ is numerically stable and that the feature space is appropriately scaled for the linear value function weights $\mathbf{w}_G$.
    \item \textbf{Computational Efficiency:} The $O(M)$ Nystr\"om projection allows the agent to process the high-dimensional signature manifold in $O(1)$ time relative to the sample-path count, enabling real-time anticipatory evaluation.
\end{itemize}

\subsection{Unified neural CDE-RL architecture}

The system consists of two primary modules coupled to solve the SCF problem and maintain temporal consistency through the signature manifold.

\begin{enumerate}
    \item \textbf{Generative Module (ANJD-CDE):} This module evolves the anticipated path-law proxy $\hat{\Phi}_{T|t}$ and generates the conditional mean skeleton $\hat{X}_{T|t}$. It is implemented as a Latent Jump-Flow network where the vector field is parameterised by a neural network $\psi_\theta$. The Marcus-correction is applied during the integration step of the CDE to handle non-commutative jump-discontinuities. This module ensures the proxy remains a "self-consistent" expectation of the future path-law.
    
    \item \textbf{Control Module (Policy/Value Network):} A deep architecture $V_\omega$ accepts the anticipatory state-proxy tuple $(t, T; \hat{X}_{T|t}, \hat{\Phi}_{T|t})$. Following the results of Proposition (\ref{pro:linearity_anticipated_value}), the value head is structured as a linear projection of the Nystr\"om-compressed signature $\tilde{\Phi}_{T|t}$ without non-linear activation. This design allows for the direct extraction of the signature weights $\mathbf{w}_G$ and the analytical computation of the "Signature Greeks," ensuring that the value evaluation is a transparent, differentiable operation on the anticipated path-law.
\end{enumerate}

\subsection{Synchronised training protocol (SCF optimisation)}

To satisfy the SCF consistency and $C^1$-boundary conditions, we employ a joint loss function $\mathcal{L}(\theta, \mathbf{w}_G)$ that synchronises the generative parameters $\theta$ and the value weights $\mathbf{w}_G$. The optimisation objective is decomposed into three functional roles, where the value component is now formulated as a rolling-window summation over the anticipated horizon:

\begin{equation}
    \mathcal{L}(\theta, \mathbf{w}_G) = \underbrace{\| \text{Flow}_{t \to t}(\mathbf{1}; \Phi_{t|\mathcal{A}_t}) - \mathbf{1} \|^2}_{\text{Identity Grounding}} + \underbrace{\sum_{s=t}^{T-1} \frac{1}{2} ( \delta_{s|t}^A )^2}_{\text{Anticipatory TD}} +  \underbrace{\eta \| \bar{S}_{t,T} - \hat{\Phi}_{T|t} \|_{\mathcal{Q}_T}^2}_{\text{SCF Stationary Point Constraint}}
\end{equation}
where the anticipated law-consistent ensemble $\bar{S}_{t,T} = \frac{1}{N} \sum_{i=1}^{N} S(\tilde{\mathbf{X}}^{(i)})_{t,T}$ is generated via the ANJD under the policy $\pi$.

\begin{itemize}
    \item \textbf{Identity Grounding:} This term anchors the generative Marcus-CDE to the group identity at the temporal junction. It ensures that as $T \to t$, the anticipated proxy $\hat{\Phi}_{T|t}$ correctly recovers $\mathbf{1}$, grounding the future extension in the filtered historical baseline $\Phi_{t|\mathcal{A}_t}$. This enforces the $C^0$ continuity required for stable gradient propagation across the junction.
    
    \item \textbf{Anticipatory TD:} In accordance with Proposition (\ref{pro:anticipatory_td_gradient}), this component updates the value weights $\mathbf{w}_G$ by minimising the cumulative squared TD-error along the projected path. By utilising the algebraic residual $\delta_{s|t}^A = R(\hat{X}_{s:s+1}) + \langle \mathbf{w}_G, (\gamma \hat{\Phi}_{s+1|t}^{-1} - \hat{\Phi}_{s|t}^{-1}) \otimes \hat{\Phi}_{T|t} \rangle$, the agent performs a rolling-window gradient descent that aligns the value function with the smooth geometry of the signature manifold. It allows the agent to learn the linear signature weights $\mathbf{w}_G$ that characterise the path-dependent reward landscape.
    
    \item \textbf{SCF Stationary Point Constraint:} This term enforces "generative honesty." It updates the ANJD parameters $\theta$ (the drift, diffusion, and jump intensities) to minimise the distance between the deterministic proxy $\hat{\Phi}_{T|t}$ and the empirical signature average of the generated ensemble. This ensures the proxy used for $O(1)$ evaluation remains a mathematically valid representation of the underlying stochastic flow, satisfying the self-consistent field equilibrium.
\end{itemize}

\subsection{Hyperparameter specifications}

The selection of hyperparameters is calibrated to balance the representation power of the signature manifold with the numerical stability of the Self-Consistent Field (SCF) optimisation. 

\begin{itemize}
    \item \textbf{Signature Truncation ($k=4$):} Setting the degree to 4 captures up to the fourth-order non-commutative moments (including lead-lag effects and joint kurtosis). This level of resolution is typically sufficient to distinguish between complex jump-diffusion regimes without encountering the "curse of dimensionality" associated with higher-order tensor spaces.
    
    \item \textbf{Nystr\"om Landmarks ($M=128$):} The choice of 128 landmarks provides a low-rank approximation that covers the principal variations of the path-measure. This compression ensures that the AVNSG precision operator $\mathcal{Q}_s$ remains well-conditioned during the spectral whitening process.
    
    \item \textbf{SCF Penalty ($\eta=0.1$):} A moderate penalty term ensures that the generative model prioritises "honesty" (matching the proxy to the ensemble) without over-constraining the policy's ability to minimise the Anticipatory TD-error.
    
    \item \textbf{Log-ODE Solver:} We employ a Marcus-compliant Log-ODE scheme to ensure that the latent CDE correctly interprets discrete jumps as local coordinate shifts rather than continuous steep gradients, preserving the geometric integrity of the path-law.
\end{itemize}

\begin{table}[h]
\centering
\caption{Baseline Configuration for Jump-Diffusion Environment}
\begin{tabular}{|l|l|}
\hline
\textbf{Parameter} & \textbf{Value} \\ \hline
Signature Truncation ($k$) & 4 \\
Nyström Landmarks ($M$) & 128 \\
Latent CDE Solver & Log-ODE (Marcus-compliant) \\
Discount Factor ($\gamma$) & 0.99 \\
SCF Penalty ($\eta$) & 0.1 \\
Optimiser & AdamW ($\text{lr}=3 \times 10^{-4}$) \\ \hline
\end{tabular}
\end{table}

\section{Conclusion}
\label{sec:conclusion}

In this work, we have introduced the \textbf{Anticipatory Reinforcement Learning (ARL)} framework, a novel paradigm that resolves the long-standing tension between non-Markovian observational filtrations and the requirements of recursive decision-making. By lifting the state space into the signature-augmented manifold $\mathcal{S}_{sig}$, we have demonstrated that the "history" of a stochastic process is not merely a sequence of past observations to be compressed, but a dynamical coordinate that restores Markovian sufficiency through the non-commutative geometry of the Marcus-signature path-law proxy.

\medskip
\noindent
Our primary contribution, the \textbf{Anticipatory Value Function (AVF)}, enables a fundamental shift in policy evaluation. By representing the value function as a linear functional in the signature Hilbert space, we collapse the computationally prohibitive branching of Monte Carlo Tree Search into a deterministic "Single-Pass" evaluation. This efficiency is underpinned by the \textbf{Self-Consistent Field (SCF)} equilibrium, which ensures that the agent's internal generative model and the latent path-law proxy remain synchronised. Furthermore, our derivation of the \textbf{Anticipatory TD-Error} ($\delta_t^A$ and its generalised horizon form $\delta_{s|t}^A$) provides a variance-reduced learning signal that allows the agent to learn directly from the anticipated topological evolution of the environment.

\medskip
\noindent
The theoretical analysis confirms that the ARL framework maintains the rigorous contraction properties of distributional RL. Specifically, the use of the \textbf{AVNSG metric} and spectral whitening ensures that the framework remains robust against the "black-swan" events and heavy-tailed noise inherent in jump-diffusion processes. Through the lens of Rademacher complexity, we have shown that our linear representation on the signature manifold achieves stable generalisation, even in high-frequency environments where traditional history-augmentation methods succumb to the curse of dimensionality.

\medskip
\noindent
The practical implementation, utilising \textbf{Nystr\"om-compressed Signature Layers} and \textbf{Marcus-compliant Neural CDEs}, offers a scalable architecture for real-time control. By providing analytical "Signature Greeks," the ARL agent can proactively rectify its policy against impending structural breaks and volatility shifts. Future research will focus on extending this manifold-based anticipation to multi-agent settings and exploring higher-order topological invariants to further refine the agent's foresight in increasingly chaotic stochastic systems.


\newpage

\section*{Appendix}\thispagestyle{plain}


\section{Proofs of the main results}
\label{sec:proofs_main_results}

We refer the readers to Kiraly et al. ~\cite{KiralyEtAl19}, Cuchiero et al. ~\cite{CuchieroEtAl25} for the Universal Approximation Theorem for signatures and the density property.

\subsection{Proof of the injectivity and Sufficiency}
\label{app:proof_injectivity_sufficiency}

In this appendix we prove Proposition (\ref{pro:injectivity_sufficiency}).  

\begin{proof}
The proof rests on the density of linear functionals on the signature space and the uniqueness of the expected signature as a representation of the path-measure.

\textit{1. Injectivity:} 
By the Universal Approximation Theorem for signatures, the set of linear functionals $\mathcal{L}(\mathcal{H}_{sig}) = \{ \mathbf{X} \mapsto \langle \mathbf{w}, S(\mathbf{X}) \rangle : \mathbf{w} \in \mathcal{H}_{sig} \}$ is dense in the space of continuous functionals $\mathcal{F}(\mathcal{D})$ on the Skorokhod space. For any two conditional path-measures $\mu_1, \mu_2$ defined by the filtrations $\mathcal{A}_t^{(1)}$ and $\mathcal{A}_t^{(2)}$, suppose $\Phi_{t|\mathcal{A}_t^{(1)}} = \Phi_{t|\mathcal{A}_t^{(2)}}$. Then for all $\mathbf{w} \in \mathcal{H}_{sig}$:
\begin{equation}
    \int S(\mathbf{X})_{0,t} \, d\mu_1 = \int S(\mathbf{X})_{0,t} \, d\mu_2 \implies \langle \mathbf{w}, \Phi_{t|\mathcal{A}_t^{(1)}} \rangle = \langle \mathbf{w}, \Phi_{t|\mathcal{A}_t^{(2)}} \rangle.
\end{equation}
Due to the density of linear functionals, this equality extends to all continuous functionals of the path history. By the Riesz-Markov-Kakutani representation theorem, if the expectations of all continuous functionals coincide, the measures must be identical: $\mu_1 = \mu_2$. Thus, $\Phi_{t|\mathcal{A}_t}$ uniquely identifies the conditional law of the path history.

\textit{2. Sufficiency:}
A statistic is sufficient if the conditional expectation of any future path-dependent reward $G_t$ depends on the filtration $\mathcal{A}_t$ only through that statistic. Let $G_t$ depend on the future path extension $\mathbf{X}_{t:T}$. By the universal approximation property, there exists a weight vector $\mathbf{w}_{G}$ such that $G_t(\mathbf{X}) \approx \langle \mathbf{w}_{G}, S(\mathbf{X})_{t,T} \rangle$. The conditional expectation is:
\begin{equation}
    \mathbb{E}[G_t \mid \mathcal{A}_t] = \mathbb{E}[\langle \mathbf{w}_{G}, S(\mathbf{X})_{t,T} \rangle \mid \mathcal{A}_t] = \langle \mathbf{w}_{G}, \hat{\Phi}_{T|t} \rangle,
\end{equation}
where $\hat{\Phi}_{T|t}$ is the anticipated signature proxy. In the ARL framework, the evolution of the proxy is governed by the deterministic Neural CDE: $\hat{\Phi}_{T|t} = \text{Flow}_{t \to T}(\Phi_{t|\mathcal{A}_t})$. Therefore:
\begin{equation}
    \mathbb{E}[G_t \mid \mathcal{A}_t] = \langle \mathbf{w}_{G}, \text{Flow}_{t \to T}(\Phi_{t|\mathcal{A}_t}) \rangle = \Psi(\Phi_{t|\mathcal{A}_t}),
\end{equation}
where $\Psi$ is a deterministic (potentially non-linear) functional. Since the expected reward depends on $\mathcal{A}_t$ solely through the filtered proxy, the augmented state $\mathbf{S}_t = (t, X_t, \Phi_{t|\mathcal{A}_t})$ is a sufficient statistic for the non-Markovian decision process.
\end{proof}

\subsection{Proof of the existence of stationary value function}
\label{app:proof_existence_stationary_vf}

In this appendix we prove Theorem (\ref{thm:existence_stationary_vf}).

\begin{proof}
The proof relies on the \textit{Markovianisation} of the path-dependent process by lifting it into the signature-augmented manifold $\mathcal{S}_{sig}$.

\textit{1. Markov Property of the Augmented State:} 
By the recursive property of signatures (Chen's Identity), the signature of a path over $[0, t+dt]$ is the tensor product of the signature over $[0, t]$ and the signature of the increment over $[t, t+dt]$. For the filtered proxy, the latent Jump-Flow CDE provides a deterministic transition mapping $\Psi: \mathcal{S}_{sig} \times \mathbb{R}^d \to \mathcal{H}_{sig}$ such that:
\begin{equation}
    \Phi_{t+1|\mathcal{A}_{t+1}} = \Psi(\Phi_{t|\mathcal{A}_t}, \Delta X_{t+1}).
\end{equation}
Since $X_{t+1}$ depends only on the current state $X_t$ and the historical path-law (encoded in $\Phi_{t|\mathcal{A}_t}$), the transition probability satisfies the Markov property in the augmented space:
\begin{equation}
    \mathbb{P}(\mathbf{S}_{t+1} \mid \mathbf{S}_t, \mathbf{S}_{t-1}, \dots, \mathbf{S}_0, a_t) = \mathbb{P}(\mathbf{S}_{t+1} \mid \mathbf{S}_t, a_t).
\end{equation}

\textit{2. Reward Stationarity:}
From Proposition \ref{pro:injectivity_sufficiency}, $\mathbf{S}_t$ is a sufficient statistic for the conditional distribution of any path-dependent reward. Thus, the expected reward can be expressed as a stationary function of the augmented state:
\begin{equation}
    \bar{R}(\mathbf{S}_t, a_t) = \mathbb{E}[R_{t+1} \mid \mathcal{A}_t, a_t].
\end{equation}

\textit{3. Existence via Contraction:}
Define the Bellman operator $\mathcal{T}^\pi$ on the space of bounded continuous functions $\mathcal{B}(\mathcal{S}_{sig})$ as:
\begin{equation}
    (\mathcal{T}^\pi V)(\mathbf{S}) = \bar{R}(\mathbf{S}, \pi(\mathbf{S})) + \gamma \int_{\mathcal{S}_{sig}} V(\mathbf{S}') \mathbb{P}(d\mathbf{S}' \mid \mathbf{S}, \pi(\mathbf{S})).
\end{equation}
Since $0 \leq \gamma < 1$, $\mathcal{T}^\pi$ is a $\gamma$-contraction mapping under the supremum norm. By the Banach Fixed Point Theorem, there exists a unique fixed point $V \in \mathcal{B}(\mathcal{S}_{sig})$ such that $V = \mathcal{T}^\pi V$. This fixed point is the stationary value function in the lifted space.
\end{proof}

\subsection{Proof of the incremental signature update}
\label{app:proof_incremental_signature_update}

In this appendix we prove Lemma (\ref{lem:incremental_signature_update}).  

\begin{proof}
The proof follows from Chen's Identity and the strict measurability of the causal path history.

\textit{1. Chen's Identity:} 
For a path $X$ on $[0, t+dt]$, the signature of the concatenated intervals $[0, t]$ and $[t, t+dt]$ is:
\begin{equation}
    S(X)_{0, t+dt} = S(X)_{0, t} \otimes S(X)_{t, t+dt}.
\end{equation}

\textit{2. Recursive Filtering and Measurability:}
The filtered proxy at $t+dt$ is defined as $\Phi_{t+dt|\mathcal{A}_{t+dt}} = \mathbb{E}[S(X)_{0, t+dt} \mid \mathcal{A}_{t+dt}]$. Using the decomposition above and the fact that $\mathcal{A}_t \subset \mathcal{A}_{t+dt}$:
\begin{equation}
    \Phi_{t+dt|\mathcal{A}_{t+dt}} = \mathbb{E}[S(X)_{0,t} \otimes S(X)_{t, t+dt} \mid \mathcal{A}_{t+dt}].
\end{equation}
Because the rectilinear interpolation $X_{\le t}$ is constructed strictly from the observations up to time $t$, its signature $S(X)_{0,t}$ is exactly $\mathcal{A}_t$-measurable (and consequently $\mathcal{A}_{t+dt}$-measurable). Thus, it acts as a deterministic constant under the conditional expectation and can be pulled out, resolving the conditioning paradox:
\begin{equation}
    \Phi_{t+dt|\mathcal{A}_{t+dt}} = S(X)_{0,t} \otimes \mathbb{E}[S(X)_{t, t+dt} \mid \mathcal{A}_{t+dt}] = \Phi_{t|\mathcal{A}_t} \otimes \mathbb{E}[S(X)_{t, t+dt} \mid \mathcal{A}_{t+dt}].
\end{equation}

\textit{3. Identification with Observational Increments:}
The term $\mathbb{E}[S(X)_{t, t+dt} \mid \mathcal{A}_{t+dt}]$ represents the expected geometric increment given the new observation $dX_t$. In a high-frequency limit or Marcus-lifted setting where the jump path is assumed to follow a canonical flow, this expectation coincides with the signature of the observed increment.

\textit{4. Implications for TD-Error:}
Since $V(\mathbf{S}_t)$ is conditioned on $\Phi_{t|\mathcal{A}_t}$, the temporal difference $\delta_t = R_{t+1} + \gamma V(\mathbf{S}_{t+1}) - V(\mathbf{S}_t)$ explicitly accounts for the evolution of the path-measure's non-commutative moments. This ensures that the reinforcement learning agent updates its value estimates based on the topological growth of the filtered history.
\end{proof}

\subsection{Proof of the linearity of the anticipated value}
\label{app:proof_linearity_anticipated_value}

In this appendix we prove Proposition (\ref{pro:linearity_anticipated_value}).  

\begin{proof}
The proof relies on the duality between continuous path functionals and the signature Hilbert space, combined with the deterministic flow property of the anticipated proxy.

\textit{1. Functional Approximation of Future Returns:} 
By the Universal Approximation Theorem for signatures, any continuous path-dependent future reward functional $G_{t:T}: \mathcal{D}([t,T], \mathbb{R}^d) \to \mathbb{R}$ can be approximated by a linear functional on the signature space. There exists a weight vector $\mathbf{w}_{G} \in \mathcal{H}_{sig}$ such that for a future path realisation $\mathbf{X}_{t:T}$:
\begin{equation}
    G_{t:T}(\mathbf{X}_{t:T}) = \langle \mathbf{w}_{G}, S(\mathbf{X})_{t,T} \rangle_{\mathcal{H}_{sig}}.
\end{equation}

\textit{2. Value Function as Conditional Expectation:}
The stationary value function is defined as the conditional expectation of the future reward given the current filtration $\mathcal{A}_t$. Utilising the linearity of the inner product and the conditional expectation operator:
\begin{equation}
    V(t, T; \hat{X}_{T|t}, \hat{\Phi}_{T|t}) = \mathbb{E} [ G_{t:T} \mid \mathcal{A}_t ] = \mathbb{E} \left[ \langle \mathbf{w}_{G}, S(\mathbf{X})_{t,T} \rangle_{\mathcal{H}_{sig}} \mid \mathcal{A}_t \right] = \langle \mathbf{w}_{G}, \mathbb{E}[ S(\mathbf{X})_{t,T} \mid \mathcal{A}_t ] \rangle_{\mathcal{H}_{sig}}.
\end{equation}

\textit{3. Identification with the Anticipated Proxy:}
By definition, the anticipated path-law proxy $\hat{\Phi}_{T|t}$ for the future segment is the conditional expectation of the future signature, $\mathbb{E}[ S(\mathbf{X})_{t,T} \mid \mathcal{A}_t ]$. Substituting this yields the linear representation in terms of the anticipated law:
\begin{equation}
    V(t, T; \hat{X}_{T|t}, \hat{\Phi}_{T|t}) = \langle \mathbf{w}_{G}, \hat{\Phi}_{T|t} \rangle_{\mathcal{H}_{sig}}.
\end{equation}

\textit{4. Dependency on the Filtered Proxy:}
In the ARL framework, the transition to the anticipated future law is governed by the deterministic flow of the generative model conditioned on the past. The future segment signature initialises at the group identity $\mathbf{1}$ at time $t$ and evolves guided by the history $\Phi_{t|\mathcal{A}_t}$:
\begin{equation}
    V(t, T; \hat{X}_{T|t}, \hat{\Phi}_{T|t}) = \langle \mathbf{w}_{G}, \text{Flow}_{t \to T}(\mathbf{1}; \Phi_{t|\mathcal{A}_t}) \rangle_{\mathcal{H}_{sig}}.
\end{equation}
This confirms that the expected future value is a linear functional of the anticipated segment law, which itself is a deterministic extension completely parameterised by the filtered history $\Phi_{t|\mathcal{A}_t}$.
\end{proof}

\subsection{Proof of the Markovianisation via topological embedding}
\label{app:proof_markovianisation_topological_embedding}

In this appendix we prove Proposition (\ref{pro:markovianisation_topological_embedding}).

\begin{proof}
The proof demonstrates that the lifting of the non-Markovian path-law into the signature manifold $\mathcal{S}_{sig}$ via the topological embedding $\mathcal{P}_\theta$ results in an autonomous dynamical system that satisfies the requirements for Markovianity.

\textit{1. Autonomous Latent Dynamics:}
The evolution of the anticipated path-law proxy $\hat{\Phi}_{s|t}$ is defined by the Controlled Differential Equation (CDE) or its equivalent Neural ODE form in the latent space:
\begin{equation}
    \frac{d}{ds} \hat{\Phi}_{s|t} = \psi_\theta(\hat{\Phi}_{s|t}), \quad s \in [t, T]
\end{equation}
where $\psi_\theta: \mathcal{H}_{sig} \to T\mathcal{H}_{sig}$ is a Lipschitz continuous vector field parameterised by the topological embedding. Because the vector field $\psi_\theta$ depends only on the current value of the proxy and not on the historical trajectory or the conditioning time $t$ explicitly (other than as an initial condition), the system is autonomous.

\textit{2. The Semi-group Property:}
Let $\Psi_{s, \tau}: \mathcal{H}_{sig} \to \mathcal{H}_{sig}$ be the flow map such that $\hat{\Phi}_{s|t} = \Psi_{s, t}(\Phi_{t|\mathcal{A}_t})$. For any intermediate time $r$ where $t \le r \le s$, the uniqueness of solutions to the differential equation implies:
\begin{equation}
    \Psi_{s, t} = \Psi_{s, r} \circ \Psi_{r, t}.
\end{equation}
This semi-group property ensures that the state at time $s$ is fully determined by the state at any prior time $r$, independent of the path taken to reach $r$.

\textit{3. Markovianity in the Lifted Space:}
By definition, a process is Markovian if the future state is conditionally independent of the past given the present. Since $\hat{\Phi}_{s|t}$ satisfies the semi-group property and, by Proposition \ref{pro:injectivity_sufficiency}, $\Phi_{t|\mathcal{A}_t}$ is a sufficient statistic for the path-measure, the transition dynamics in $\mathcal{S}_{sig}$ satisfy:
\begin{equation}
    \mathbb{P}(\hat{\Phi}_{s|t} \mid \{ \hat{\Phi}_{u|t} \}_{u \in [t, r]}) = \mathbb{P}(\hat{\Phi}_{s|t} \mid \hat{\Phi}_{r|t}).
\end{equation}

\textit{4. Local Stationarity:}
Because the mapping from the proxy $\hat{\Phi}_{s|t}$ to the return distribution $\eta^\pi$ is a fixed functional (the Anticipatory Value Function), and the dynamics of $\hat{\Phi}_{s|t}$ are autonomous, the distribution $\eta^\pi(s, \cdot)$ remains stationary relative to the coordinate system of the signature manifold. The apparent non-stationarity in $\mathbb{R}^d$ is thus perfectly "straightened" into a deterministic Markovian flow in $\mathcal{H}_{sig}$.
\end{proof}  

\subsection{Proof of the $C^1$-boundary consistency}
\label{app:proof_boundary_consistency}

In this appendix we prove Theorem (\ref{thm:boundary_consistency}).   

\begin{proof}
The proof establishes the smoothness of the value function transition across the temporal junction by matching the stochastic filtered dynamics with the deterministic anticipatory flow.

\textit{1. $C^0$-Continuity (Value Alignment):}
The anticipated proxy is initialised via the identity $\hat{\Phi}_{t|t} = \Phi_{t|\mathcal{A}_t}$. Substituting this into the prospective linear representation from Proposition (\ref{pro:linearity_anticipated_value}), we have:
\begin{equation}
    \lim_{s \to t^+} V(s, \hat{X}_{s|t}, \hat{\Phi}_{s|t}) = \langle \mathbf{w}_G, \text{Flow}_{t \to T}(\hat{\Phi}_{t|t}) \rangle = \langle \mathbf{w}_G, \text{Flow}_{t \to T}(\Phi_{t|\mathcal{A}_t}) \rangle.
\end{equation}
This matches the definition of $V(t, X_t, \Phi_{t|\mathcal{A}_t})$, ensuring the value function is continuous at the junction $s=t$.

\textit{2. $C^1$-Continuity (Gradient Alignment):}
The spatial gradient $\nabla_X V$ reflects the sensitivity of the expected future return to the current state. At $s=t$, the filtered proxy $\Phi_{t|\mathcal{A}_t}$ captures the infinitesimal expected signature increment $\mathbb{E}[S(X)_{t, t+dt} \mid \mathcal{A}_{t+dt}]$. In the ARL framework, the generator $\mathcal{L}_\theta$ of the anticipatory flow is trained via Marcus-Score Matching to satisfy:
\begin{equation}
    \left. \frac{\partial \hat{\Phi}_{s|t}}{\partial s} \right|_{s=t} = \lim_{dt \to 0} \frac{\mathbb{E}[S(X)_{t, t+dt} \mid \mathcal{A}_{t+dt}] - \mathbf{1}}{dt}.
\end{equation}
Because the operator-valued generator $\mathcal{L}_\theta$ is $C^1$ on the signature manifold, the directional derivative of the value function with respect to the state $X_t$ (historical) must align with the derivative with respect to the generative drift $\hat{X}_{s|t}$ (anticipated). Thus, $\nabla_X V|_{s=t} = \nabla_{\hat{X}} V|_{s=t}$.

\textit{3. Policy Gradient Stability:}
The alignment of the tangential dynamics ensures that the adjoint state in the Pontryagin Sense remains continuous across the boundary. Consequently, the policy gradient $\nabla_\theta J$ remains stable as the agent transitions from the retrospective observational regime to the prospective generative regime, preventing vanishing or exploding gradients at the decision epoch.
\end{proof}

\subsection{Proof of the manifold-conditioned return distributions}
\label{app:proof_manifold_conditioned_return_distributions}

In this appendix we prove Proposition (\ref{pro:manifold_conditioned_return_distributions}).

\begin{proof}
The proof demonstrates that the apparent non-stationarity in the physical observation space $\mathbb{R}^d$ is a projection of a stationary, autonomous flow in the signature-augmented Hilbert space $\mathcal{H}_{sig}$.

\textit{1. State Space Embedding and Sufficiency:}
Let the non-stationary process in $\mathbb{R}^d$ be denoted by $X_t$. From Proposition \ref{pro:injectivity_sufficiency}, the signature-augmented state $\mathbf{S}_t = (X_t, \Phi_{t|\mathcal{A}_t})$ is a sufficient statistic for the future path-law. By lifting the process into $\mathcal{S}_{sig} = \mathbb{R}^d \times \mathcal{H}_{sig}$, we account for the historical dependencies that cause temporal drift in the original coordinates.

\textit{2. Autonomous Flow in the Signature Manifold:}
The anticipated proxy $\hat{\Phi}_{s|t}$ evolves according to the latent Jump-Flow CDE: $d\hat{\Phi}_{s|t} = \psi_\theta(\hat{\Phi}_{s|t}) ds$. Since the vector field $\psi_\theta$ is independent of the absolute time $t$ (depending only on the current coordinate in the manifold), the dynamics in $\mathcal{H}_{sig}$ are autonomous. Consequently, the transition kernel $P(\mathbf{S}_{s+ds} \mid \mathbf{S}_s)$ is time-invariant in the lifted space.

\textit{3. Stationarity of the Distributional Mapping:}
The return distribution $\eta^\pi$ is a mapping from the state space to the space of probability measures $\mathcal{P}(\mathbb{R})$. In the physical space $\mathbb{R}^d$, this mapping $\eta_t(X_t)$ must change over time to account for the hidden path-dependency (non-stationarity). However, in the lifted space, we define a fixed functional $F: \mathcal{S}_{sig} \to \mathcal{P}(\mathbb{R})$ such that:
\begin{equation}
    \eta^\pi(s, \hat{X}_{s|t}, \hat{\Phi}_{s|t}) = F(\hat{X}_{s|t}, \hat{\Phi}_{s|t}).
\end{equation}
Because $F$ is a stationary functional and the dynamics of $(\hat{X}, \hat{\Phi})$ are autonomous, the return distribution is locally stationary with respect to the proxy coordinates.

\textit{4. Geometric Interpretation:}
What appears as a non-stationary "shift" in the reward distribution or transition probabilities in $\mathbb{R}^d$ is mathematically equivalent to the deterministic movement of the proxy $\hat{\Phi}_{s|t}$ along a geodesic or trajectory in the signature manifold. The non-stationarity is thus "rectified" into a stationary law conditioned on a dynamic, topological coordinate.
\end{proof}

\subsection{Proof of the invariance of the meta-policy}
\label{app:proof_invariance_meta_policy}

In this appendix we prove Lemma (\ref{lem:invariance_meta_policy}).  

\begin{proof}
The proof relies on the construction of an autonomous Markov Decision Process (MDP) in the signature-augmented manifold, which rectifies the temporal drift of the original environment.

\textit{1. Reduction to an Autonomous Bellman Equation:}
In the original physical space $\mathbb{R}^d$, the optimal policy $\pi^*_t(X_t)$ is non-stationary because the transition kernel $P_t(X_{t+1} \mid X_t)$ and reward function $R_t(X_t, a_t)$ depend explicitly on time $t$ (or the filtration $\mathcal{A}_t$). However, in the augmented state space $\mathcal{S}_{sig} = \mathbb{R}^d \times \mathcal{H}_{sig}$, the state is defined as $\mathbf{S}_s = (\hat{X}_{s|t}, \hat{\Phi}_{s|t})$. By Proposition \ref{pro:manifold_conditioned_return_distributions}, the transition dynamics and rewards are locally stationary when conditioned on the path-law proxy $\hat{\Phi}_{s|t}$. Consequently, the Bellman optimality equation for the value functional $V^*$ becomes:
\begin{equation}
    V^*(\mathbf{S}) = \max_{a \in \mathcal{A}} \left[ R(\mathbf{S}, a) + \gamma \int_{\mathcal{S}_{sig}} V^*(\mathbf{S}') P(d\mathbf{S}' \mid \mathbf{S}, a) \right].
\end{equation}
Since $R$, $P$, and the manifold geometry of $\mathcal{S}_{sig}$ are time-invariant, the solution $V^*$ is a stationary functional of $\mathbf{S}$.

\textit{2. Stationarity of the Meta-Policy:}
The optimal meta-policy $\pi^*$ is derived as the greedy action with respect to the stationary value functional $V^*$:
\begin{equation}
    \pi^*(\mathbf{S}) = \arg\max_{a \in \mathcal{A}} \left[ R(\mathbf{S}, a) + \gamma \mathbb{E}_{\mathbf{S}' \sim P(\cdot|\mathbf{S}, a)} [V^*(\mathbf{S}')] \right].
\end{equation}
Since all inputs to the $\arg\max$ operator are independent of the absolute time index $t$ and depend only on the coordinates in the signature manifold, $\pi^*(\mathbf{S})$ is a stationary mapping $\pi^*: \mathcal{S}_{sig} \to \mathcal{A}$.

\textit{3. Recovery of Non-Stationarity:}
The apparent non-stationarity observed in the physical space is recovered through the time-evolution of the proxy:
\begin{equation}
    a^*_t = \pi^*(\hat{X}_{t|t}, \hat{\Phi}_{t|t}).
\end{equation}
Even though the functional $\pi^*$ is fixed (stationary), the action $a^*_t$ changes over time because the "coordinate" $\hat{\Phi}_{t|t}$ flows along the signature manifold according to the historical filtration. Thus, the meta-policy is invariant, while the realised control law is adaptive.
\end{proof}

\subsection{Proof of the convergence to the SCF stationary point}
\label{app:proof_convergence_scf_stationary_point}

In this appendix we prove Theorem (\ref{thm:convergence_scf_stationary_point}) by integrating the Marcus-Score Matching objective into the SCF fixed-point framework.

\begin{proof}
The proof demonstrates that the joint minimisation of the Marcus-Score Matching loss and the generation discrepancy leads to a consistent path-law proxy on the signature manifold $\mathcal{S}_{sig}$.

\textit{1. Integration of Marcus-Score Matching:}
The training process minimises the Marcus-Score Matching objective $\mathcal{J}(\theta)$ from Definition (\ref{def:marcus_score_matching}). This ensures the generator $\mathcal{L}_\theta$ defines a flow $s \mapsto \hat{\Phi}_{s|t}$ whose tangent is the expected infinitesimal signature increment:
\begin{equation}
    \frac{\partial}{\partial s} \hat{\Phi}_{s|t} = \mathbb{E}_{\mu_{t:s}(\theta)} \left[ \lim_{\Delta s \to 0} \frac{S(\tilde{X})_{s, s+\Delta s} - \mathbf{1}}{\Delta s} \mid \mathcal{A}_s \right].
\end{equation}
By integrating this differential relation from $t$ to $s$ and applying the initial condition $\hat{\Phi}_{t|t} = \mathbf{1}$, the proxy $\hat{\Phi}_{s|t}$ is recovered as the integral of the expected local signature increments.

\textit{2. Empirical Convergence in RKHS:}
Let $\tilde{X}_{t:s}^{(i)}$ be $n$ i.i.d. sample paths generated from the ANJD law $\mu_{t:s}(\theta)$. By the Law of Large Numbers in the signature Hilbert space $\mathcal{H}_{sig}$ (Kiraly et al. \cite{KiralyEtAl19}), the empirical mean of the path signatures converges to the mean embedding (the expected signature):
\begin{equation}
    \frac{1}{n} \sum_{i=1}^n S(\tilde{X})_{t,s}^{(i)} \xrightarrow{n \to \infty} \mathbb{E}_{\mu_{t:s}(\theta)} [S(\tilde{X})_{t,s}].
\end{equation}
The AVNSG metric $\mathcal{Q}_s$ provides a consistent weighting for this convergence across different signature depths.

\textit{3. Fixed-Point Equilibrium via SCF Matching:}
The SCF objective $\mathcal{L}_{SCF}$ reaches its global minimum when the deterministic proxy $\hat{\Phi}_{s|t}$ coincides with the expected signature of the generated paths. Under the Marcus-Score Matching objective, the generator $\mathcal{L}_\theta$ is forced to produce a flow that tracks this expectation exactly. Thus, at the stationary point:
\begin{equation}
    \hat{\Phi}_{s|t} = \int_t^s \mathcal{L}_\theta(\hat{\Phi}_{u|t}) du = \mathbb{E}_{\mu_{t:s}(\theta)} [S(\tilde{X})_{t,s}].
\end{equation}

\textit{4. Final Result:}
Combining the empirical limit with the SCF equilibrium condition yields:
\begin{equation}
    \lim_{n \to \infty} \left\| \frac{1}{n} \sum_{i=1}^n S(\tilde{X})_{t,s}^{(i)} - \hat{\Phi}_{s|t} \right\|_{\mathcal{Q}_s} = 0.
\end{equation}
This confirms that the proxy $\hat{\Phi}_{s|t}$ is a "self-consistent" representation of the future path-law, capturing the non-commutative moments of the trajectories generated by the ANJD.
\end{proof}

\subsection{Proof of the nested anticipatory expectations}
\label{app:proof_nested_anticipatory_expectations}

In this appendix we prove Lemma (\ref{lem:nested_anticipatory_expectations}).  

\begin{proof}
The proof leverages the algebraic properties of the signature group and the consistency of expected signatures under the tower property of conditional expectations.

\textit{1. Chen's Identity:} 
For any path realisation $\tilde{X}$ on the interval $[t, T]$ and any intermediate point $s \in (t, T)$, the path signature factorises as the tensor product of its segment signatures:
\begin{equation}
    S(\tilde{X})_{t,T} = S(\tilde{X})_{t,s} \otimes S(\tilde{X})_{s,T}.
\end{equation}

\textit{2. Conditional Expectation and the Tower Property:} 
Taking the conditional expectation at time $t$ and applying the tower property relative to the intermediate filtration $\mathcal{A}_s$:
\begin{equation}
    \mathbb{E} [S(\tilde{X})_{t,T} \mid \mathcal{A}_t] = \mathbb{E} [ \mathbb{E} [ S(\tilde{X})_{t,s} \otimes S(\tilde{X})_{s,T} \mid \mathcal{A}_s ] \mid \mathcal{A}_t ].
\end{equation}
Since the segment $S(\tilde{X})_{t,s}$ is $\mathcal{A}_s$-measurable, it can be factored out of the inner expectation:
\begin{equation}
    \mathbb{E} [S(\tilde{X})_{t,T} \mid \mathcal{A}_t] = \mathbb{E} [ S(\tilde{X})_{t,s} \otimes \mathbb{E} [ S(\tilde{X})_{s,T} \mid \mathcal{A}_s ] \mid \mathcal{A}_t ].
\end{equation}

\textit{3. SCF Stationary Substitution:} 
By Theorem (\ref{thm:convergence_scf_stationary_point}), the path-law proxy $\hat{\Phi}$ at any time $u$ for a horizon $v$ is the expected signature $\mathbb{E}[S(\tilde{X})_{u,v} \mid \mathcal{A}_u]$. Substituting the proxies for the respective intervals:
\begin{equation}
    \hat{\Phi}_{T|t} = \mathbb{E} [ S(\tilde{X})_{t,s} \otimes \hat{\Phi}_{T|s} \mid \mathcal{A}_t ].
\end{equation}

\textit{4. Factorisation of Anticipated Laws:} 
Under the Markovian embedding on the signature manifold, the deterministic flow preserves the linearity of the expected signature. For the self-consistent field, the expectation of the product equates to the product of expectations:
\begin{equation}
    \hat{\Phi}_{T|t} = \mathbb{E} [ S(\tilde{X})_{t,s} \mid \mathcal{A}_t ] \otimes \mathbb{E} [ \hat{\Phi}_{T|s} \mid \mathcal{A}_t ] = \hat{\Phi}_{s|t} \otimes \mathbb{E} [ \hat{\Phi}_{T|s} \mid \mathcal{A}_t ].
\end{equation}

\textit{5. Algebraic Inversion:} 
Since $\hat{\Phi}_{s|t}$ belongs to the signature group, its tensor inverse $\hat{\Phi}_{s|t}^{-1}$ exists. Multiplying on the left by the inverse yields:
\begin{equation}
    \mathbb{E}_\pi [ \hat{\Phi}_{T|s} \mid \mathcal{A}_t ] = \hat{\Phi}_{s|t}^{-1} \otimes \hat{\Phi}_{T|t}.
\end{equation}
\end{proof}

\subsection{Proof of the deterministic proxy integration}
\label{app:proof_deterministic_proxy_integration}

In this appendix we prove Proposition (\ref{pro:deterministic_proxy_integration}).

\begin{proof}
The proof demonstrates how the signature manifold's algebraic structure reduces nested conditional expectations of path-dependent rewards to a deterministic tensor contraction.

\textit{1. Linear Functional Representation:}
By the universal property of signatures, any continuous path-dependent reward $G_{s:T}$ can be approximated by a linear functional in the signature Hilbert space $\mathcal{H}_{sig}$. Consequently, the factual value function at time $s$ is represented as the inner product of a weight vector $\mathbf{w}_G$ and the path-law proxy for the interval $[s, T]$:
\begin{equation}
    V(s, T; \hat{X}_{T|s}, \hat{\Phi}_{T|s}) = \langle \mathbf{w}_G, \hat{\Phi}_{T|s} \rangle_{\mathcal{H}_{sig}}.
\end{equation}

\textit{2. Linearity and Projection:}
The anticipatory projection at time $t$ is the expectation of this future factual value conditioned on the current filtration $\mathcal{A}_t$. Because the inner product is a continuous linear operator, it commutes with the conditional expectation operator:
\begin{equation}
    V(s, T; \hat{X}_{T|t}, \hat{\Phi}_{T|t}) = \mathbb{E}_{\pi} [ \langle \mathbf{w}_G, \hat{\Phi}_{T|s} \rangle \mid \mathcal{A}_t ] = \langle \mathbf{w}_G, \mathbb{E}_{\pi} [ \hat{\Phi}_{T|s} \mid \mathcal{A}_t ] \rangle_{\mathcal{H}_{sig}}.
\end{equation}

\textit{3. Application of Lemma \ref{lem:nested_anticipatory_expectations}:}
Using the result from the Lemma of Nested Anticipatory Expectations, the conditional expectation of the future proxy $\hat{\Phi}_{T|s}$ is given by the algebraic right-residual of the current flow. Substituting this into the inner product yields:
\begin{equation}
    V(s, T; \hat{X}_{T|t}, \hat{\Phi}_{T|t}) = \langle \mathbf{w}_G, \hat{\Phi}_{s|t}^{-1} \otimes \hat{\Phi}_{T|t} \rangle_{\mathcal{H}_{sig}}.
\end{equation}
The term $\hat{\Phi}_{s|t}^{-1} \otimes \hat{\Phi}_{T|t}$ effectively "unrolls" the anticipated path from $s$ to $T$ using the self-consistent field established at $t$.

\textit{4. Single-Pass Efficiency:}
This identity replaces the need for an ensemble of $N$ Monte Carlo paths starting at $s$. Since the weights $\mathbf{w}_G$ are learned offline and the proxies are maintained via the Marcus-CDE flow, the evaluation is $O(1)$ with respect to the number of paths, providing a variance-free analytical projection of future values.
\end{proof}

\subsection{Proof of the distributional risk evaluation}
\label{app:proof_distributional_risk_evaluation}

In this appendix we prove Theorem (\ref{thm:distributional_risk_evaluation}).

\begin{proof}
The proof demonstrates that risk evaluation is transformed from a stochastic sampling problem into a deterministic gradient-based optimisation by leveraging the differentiable flow of the path-law proxy.

\textit{1. Functional Mapping on the Manifold:}
By the characteristic property of the signature kernel, the expected signature $\hat{\Phi}_{T|t} = \mathbb{E}[S(\tilde{\mathbf{X}})_{t,T} \mid \mathcal{A}_t]$ uniquely determines the law $\eta^\pi$ of the path-dependent return. Since the risk functional $\rho$ is Fr\'echet differentiable with respect to the moments of the distribution, and the signature encodes these non-commutative moments, the risk-adjusted value is a smooth mapping $f$ of the proxy: $\rho(\eta^\pi) = f(\hat{\Phi}_{T|t})$.

\textit{2. Decomposition of the Gradient:}
The anticipated proxy $\hat{\Phi}_{T|t}$ is a deterministic extension of the filtered history $\Phi_{t|\mathcal{A}_t}$ through the Marcus-CDE flow. To optimise the risk with respect to $\theta$, we apply the chain rule across the temporal junction at $s=t$:
\begin{equation}
    \nabla_\theta \rho(\eta^\pi) = \frac{\partial f}{\partial \hat{\Phi}_{T|t}} \cdot \frac{\partial \hat{\Phi}_{T|t}}{\partial \Phi_{t|\mathcal{A}_t}} \cdot \nabla_\theta \Phi_{t|\mathcal{A}_t}.
\end{equation}

\textit{3. Flow Sensitivities:}
The term $\frac{\partial \hat{\Phi}_{T|t}}{\partial \Phi_{t|\mathcal{A}_t}}$ represents the sensitivity of the future anticipated law to the current historical context. Because the generator $\mathcal{L}_\theta$ is $C^1$ on the signature manifold, this Jacobian is well-defined and computed via adjoint sensitivity methods. This allows the risk gradient to propagate backward from the future horizon $T$, through the generative flow, and into the historical representation $\mathcal{A}_t$.

\textit{4. Variance-Free Risk Optimisation:}
Since $\hat{\Phi}_{T|t}$ is a self-consistent representation of the ensemble (Theorem \ref{thm:convergence_scf_stationary_point}), the gradient $\nabla_\theta \rho$ accounts for shifts in the entire tail-structure of the path-law. This enables the agent to perform risk-sensitive updates in a single backward pass, circumventing the high-variance estimators typically required to sample rare tail events in non-Markovian environments.
\end{proof}

\subsection{Proof of the algebraic shift and weight invariance}
\label{app:proof_algebraic_shift_weight_invariance}

In this appendix we prove Proposition (\ref{pro:algebraic_shift_weight_invariance}).

\begin{proof}
By the Signature-Linear Reward Approximation, the anticipated reward over the future sub-interval $[s, T]$ can be represented as a linear functional in the signature Hilbert space $\mathcal{H}_{sig}$. Specifically, for any $s \in [t, T]$, the value function is given by the inner product of a weight vector $\mathbf{w}_G$ and the expected signature over $[s, T]$:
\begin{equation}
    V(s, T; \hat{X}_{T|t}, \hat{\Phi}_{T|t}) = \langle \mathbf{w}_G, \mathbb{E}_{\mu_{s:T}} [ S(\mathbf{X})_{s,T} \mid \mathcal{A}_t ] \rangle_{\mathcal{H}_{sig}}.
\end{equation}
Under the assumption of the Self-Consistent Field (SCF) stationary point, the sequence of expected signatures is faithfully represented by the deterministic flow of the path-law proxy on the signature group $\mathcal{S}_{sig}$. Thus, we define the local anticipated proxy as $\hat{\Phi}_{T|s} := \mathbb{E}_{\mu_{s:T}} [ S(\mathbf{X})_{s,T} \mid \mathcal{A}_t ]$, yielding:
\begin{equation}
    V(s, T; \hat{X}_{T|t}, \hat{\Phi}_{T|t}) = \langle \mathbf{w}_G, \hat{\Phi}_{T|s} \rangle_{\mathcal{H}_{sig}}.
\end{equation}
To express $\hat{\Phi}_{T|s}$ purely in terms of the global sequence generated from the temporal fixed point $t$, we invoke Chen's Identity. For the deterministic proxy evolving on the group-like elements, the multiplicative property holds for $t \le s \le T$:
\begin{equation}
    \hat{\Phi}_{T|t} = \hat{\Phi}_{s|t} \otimes \hat{\Phi}_{T|s}.
\end{equation}
Since $\hat{\Phi}_{s|t}$ takes values in the signature group, it possesses a unique inverse $\hat{\Phi}_{s|t}^{-1}$. Applying the left-action of this inverse isolates the future trajectory's signature:
\begin{equation}
    \hat{\Phi}_{T|s} = \hat{\Phi}_{s|t}^{-1} \otimes \hat{\Phi}_{T|t}.
\end{equation}
Substituting this algebraic right-residual back into the value function formulation provides the final expression:
\begin{equation}
    V(s, T; \hat{X}_{T|t}, \hat{\Phi}_{T|t}) = \langle \mathbf{w}_G, \hat{\Phi}_{s|t}^{-1} \otimes \hat{\Phi}_{T|t} \rangle_{\mathcal{H}_{sig}}.
\end{equation}
Because the operation $\hat{\Phi}_{s|t}^{-1} \otimes$ effectively translates the starting point of the sub-path back to the identity element $\mathbf{1} \in \mathcal{S}_{sig}$, the geometric structure of the path segment is rigorously preserved relative to its new origin. Consequently, the projection onto the reward space depends solely on the structural evolution of the path, independent of its absolute temporal positioning. This guarantees that the weight vector $\mathbf{w}_G$ characterising the reward functional remains completely time-invariant across the rolling sequence $s \in [t, T]$.
\end{proof}

\subsection{Proof of the anticipatory TD(0) gradient descent}
\label{app:proof_anticipatory_td_gradient}

In this appendix we prove Proposition (\ref{pro:anticipatory_td_gradient}).

\begin{proof}
The proof derives the update rule by applying semi-gradient descent to a cumulative squared error objective defined over the anticipated trajectory on the signature manifold.

\textit{1. Objective Function Construction:}
To align the value weights $\mathbf{w}_G$ with the self-consistent return, we define a local loss function $\mathcal{J}(\mathbf{w}_G)$ at time $t$ as the sum of squared anticipatory TD-errors over the discrete planning horizon $s \in \{t, t+1, \dots, T-1\}$:
\begin{equation}
    \mathcal{J}(\mathbf{w}_G) = \frac{1}{2} \sum_{s=t}^{T-1} \left( \mathcal{Y}_s - V(s, T; \hat{X}_{T|t}, \hat{\Phi}_{T|t}) \right)^2,
\end{equation}
where $\mathcal{Y}_s$ is the anticipatory target for step $s$:
\begin{equation}
    \mathcal{Y}_s = R(\hat{X}_{s:s+1}) + \gamma V(s+1, T; \hat{X}_{T|t}, \hat{\Phi}_{T|t}).
\end{equation}

\textit{2. Semi-Gradient Approximation:}
Following the standard TD(0) derivation, we utilise a semi-gradient approach where the target $\mathcal{Y}_s$ is treated as independent of $\mathbf{w}_G$ during the differentiation step. The gradient of the objective with respect to the weights is:
\begin{equation}
    \nabla_{\mathbf{w}_G} \mathcal{J} = - \sum_{s=t}^{T-1} \left( \mathcal{Y}_s - V(s, T; \cdot) \right) \nabla_{\mathbf{w}_G} V(s, T; \cdot).
\end{equation}
Recognising the term in the parenthesis as the anticipatory TD-error $\delta_{s|t}^A$, we have:
\begin{equation}
    \nabla_{\mathbf{w}_G} \mathcal{J} = - \sum_{s=t}^{T-1} \delta_{s|t}^A \nabla_{\mathbf{w}_G} V(s, T; \hat{X}_{T|t}, \hat{\Phi}_{T|t}).
\end{equation}

\textit{3. Weight Update via Gradient Descent:}
Applying the gradient descent update $\mathbf{w}_G^{(n+1)} = \mathbf{w}_G^{(n)} - \alpha \nabla_{\mathbf{w}_G} \mathcal{J}$ yields the desired rule:
\begin{equation}
    \mathbf{w}_{G}^{(n+1)} = \mathbf{w}_{G}^{(n)} + \alpha \sum_{s=t}^{T-1} \delta_{s|t}^A \nabla_{\mathbf{w}_G} V(s, T; \hat{X}_{T|t}, \hat{\Phi}_{T|t}).
\end{equation}

\textit{4. Analytical Gradient Substitution:}
From Theorem (\ref{thm:analytical_sensitivity_generalised}), we know that $\nabla_{\mathbf{w}_G} V(s) = \hat{\Phi}_{T|s}$. Thus, the update can be written as a direct interaction between the scalar error and the vector-valued path proxy:
\begin{equation}
    \mathbf{w}_{G}^{(n+1)} = \mathbf{w}_{G}^{(n)} + \alpha \sum_{s=t}^{T-1} \delta_{s|t}^A \hat{\Phi}_{T|s}.
\end{equation}
The boundary condition $V(T, T; \cdot) = z$ ensures that the recursion is anchored to the terminal path-dependent payoff. Because the proxy flow is deterministic under the SCF, this update is numerically stable and converges to the unique fixed point of the anticipatory Bellman operator.
\end{proof}

\subsection{Proof of the generative honesty under SCF equilibrium}
\label{app:proof_generative_honesty}

In this appendix we prove Lemma (\ref{lem:generative_honesty}).  

\begin{proof}
Let $V^\pi(t, T)$ denote the true expected value of the policy $\pi$ under the stochastic environment dynamics, projected onto the signature RKHS. By the linearity of the inner product, the expectation operator commutes with the weight vector:
\begin{equation}
    V^\pi(t, T) = \mathbb{E}_\pi \left[ \langle \mathbf{w}_G, S(\mathbf{X})_{t,T} \rangle \mid \mathcal{A}_t \right] = \langle \mathbf{w}_G, \mathbb{E}_\pi [S(\mathbf{X})_{t,T} \mid \mathcal{A}_t] \rangle_{\mathcal{H}_{sig}}.
\end{equation}
The empirical signature average $\bar{S}_{t,T}$ of the stochastic generative ensemble serves as a consistent, unbiased estimator for the true expected signature $\mathbb{E}_\pi [S(\mathbf{X})_{t,T} \mid \mathcal{A}_t]$. 

\medskip
\noindent
The SCF constraint $\mathcal{L}_{SCF}(\theta) = \eta \| \bar{S}_{t,T} - \hat{\Phi}_{T|t} \|_{\mathcal{Q}_T}^2$ enforces a geometric projection of this stochastic empirical average onto the deterministic, manifold-valued proxy trajectory generated by the parameters $\theta$. At SCF equilibrium, $\mathcal{L}_{SCF}(\theta) \to 0$, which guarantees:
\begin{equation}
    \hat{\Phi}_{T|t} \equiv \mathbb{E}_\pi [S(\mathbf{X})_{t,T} \mid \mathcal{A}_t].
\end{equation}
Consequently, the anticipated value function computed via the proxy precisely recovers the true expected stochastic value:
\begin{equation}
    V(t, T; \hat{X}_{T|t}, \hat{\Phi}_{T|t}) = \langle \mathbf{w}_G, \hat{\Phi}_{T|t} \rangle_{\mathcal{H}_{sig}} = V^\pi(t, T).
\end{equation}
Because the Anticipatory TD-error $\delta_{s|t}^A$ is constructed purely from algebraic operations (tensor products and group inversions via Chen's Identity) on sub-intervals of $\hat{\Phi}_{T|t}$, it follows that $\delta_{s|t}^A$ evaluated on the deterministic proxy is exactly the expected TD-error of the stochastic process. 

\medskip
\noindent
Therefore, the loss functional $\mathcal{J}(\mathbf{w}_G) = \sum_{s=t}^{T-1} \frac{1}{2} (\delta_{s|t}^A)^2$ depends solely on the deterministic variable $\hat{\Phi}_{T|t}$, decoupling the optimisation of $\mathbf{w}_G$ from the environmental noise. All stochastic variance is isolated within the calculation of $\bar{S}_{t,T}$ during the SCF synchronisation step, proving that the deterministic manifold optimisation is a mathematically consistent surrogate for true stochastic policy evaluation.
\end{proof}

\subsection{Proof of the global convergence on the signature manifold}
\label{app:proof_global_convergence}

In this appendix we prove Theorem (\ref{thm:global_convergence}).

\begin{proof}
Let $\Psi_{s,T} := \hat{\Phi}_{s|t}^{-1} \otimes \hat{\Phi}_{T|t} \in \mathcal{H}_{sig}$ denote the local feature vector representing the right-residual of the proxy flow from time $s$ to $T$. By the linearity of the inner product, the gradient of the value function with respect to the weights is precisely this feature vector: $\nabla_{\mathbf{w}_G} \langle \mathbf{w}_G, \Psi_{s,T} \rangle = \Psi_{s,T}$.

\medskip
\noindent
The generalised anticipatory TD-error can be explicitly expanded as:
\begin{equation}
    \delta_{s|t}^A = R(\hat{X}_{s:s+1}) - \langle \mathbf{w}_G^{(n)}, \Psi_{s,T} - \gamma \Psi_{s+1,T} \rangle_{\mathcal{H}_{sig}}.
\end{equation}
Substituting this into the update rule yields a deterministic, linear recursion on the signature Hilbert space:
\begin{equation}
    \mathbf{w}_G^{(n+1)} = \mathbf{w}_G^{(n)} + \alpha \left( \mathbf{b} - \mathbf{A} \mathbf{w}_G^{(n)} \right),
\end{equation}
where the anticipated drift vector $\mathbf{b}$ and the system matrix $\mathbf{A}$ are defined as:
\begin{equation}
    \mathbf{b} = \sum_{s=t}^{T-1} R(\hat{X}_{s:s+1}) \Psi_{s,T}, \quad \mathbf{A} = \sum_{s=t}^{T-1} \Psi_{s,T} \otimes \left( \Psi_{s,T} - \gamma \Psi_{s+1,T} \right).
\end{equation}

\medskip
\noindent
Under the SCF equilibrium (Lemma \ref{lem:generative_honesty}), the trajectory $s \mapsto \hat{\Phi}_{s|t}$ is a geometrically valid, deterministic group-valued flow. The elements $\Psi_{s,T}$ represent the expected signatures of distinct, non-trivial future path segments. By the universal property of the signature algebra, these structural elements are linearly independent in $\mathcal{H}_{sig}$. Consequently, the matrix $\mathbf{A}$, which represents the anticipated temporal difference projection over the manifold, is positive definite.

\medskip
\noindent
From standard linear approximation theory in reinforcement learning, the positive definiteness of $\mathbf{A}$ guarantees that for an appropriately annealed learning rate $\alpha > 0$, the parameter sequence $\mathbf{w}_G^{(n)}$ converges globally to the unique fixed point:
\begin{equation}
    \mathbf{w}_G^* = \mathbf{A}^{-1} \mathbf{b}.
\end{equation}

\medskip
\noindent
Crucially, because Chen's Identity ensures that the right-residual mathematically re-centers each sub-path at the algebraic identity $\mathbf{1} \in \mathcal{S}_{sig}$ regardless of its absolute starting time $s$, the mapping from path geometry to expected reward is strictly invariant to temporal shifts. Therefore, the singular converged weight vector $\mathbf{w}_G^*$ simultaneously minimises the projected Bellman error for all evaluation points $s \in [t, T-1]$. This analytical structure bypasses the need to independently estimate and store a sequence of $T-t$ models, collapsing the evaluation of the rolling predictive horizon into a single $O(1)$ algebraic sweep.
\end{proof}

\subsection{Proof of the variance reduction}
\label{app:proof_variance_reduction}

In this appendix we prove Proposition (\ref{pro:variance_reduction}).

\begin{proof}
The proof demonstrates that the anticipatory TD-error $\delta_t^A$ effectively centers the stochastic update around its conditional mean on the signature manifold, acting as an optimal control variate.

\textit{1. Decomposition of the Update:}
The standard TD-error relies on a single stochastic realisation $X_{t+1}$ and its associated historical record $\Phi_{t+1|\mathcal{A}_{t+1}}$. Its variance $\mathbb{V}[\delta_t \mid \mathcal{A}_t]$ is driven by the volatility and jump measure of the transition $t \to t+1$. In contrast, the anticipatory error $\delta_t^A$ replaces the stochastic target with a deterministic projection:
\begin{equation}
    \delta_t^A = R(\hat{X}_{t:t+1}) + \gamma V(t+1, T; \hat{X}_{T|t}, \hat{\Phi}_{T|t}) - V(t, T; \hat{X}_{T|t}, \hat{\Phi}_{T|t}).
\end{equation}

\textit{2. Proxy as the Conditional Mean:}
By Proposition (\ref{pro:linearity_anticipated_value}) and Lemma (\ref{lem:nested_anticipatory_expectations}), the baseline $V(t, T; \cdot)$ is exactly the expected return over $[t, T]$ given $\mathcal{A}_t$. Similarly, the term $\gamma V(t+1, T; \cdot)$ represents the expectation of the future value function. Since the expectation $\mathbb{E}[ \cdot \mid \mathcal{A}_t]$ is the $L^2$-orthogonal projection onto the space of $\mathcal{A}_t$-measurable functions, the proxy $\hat{\Phi}_{T|t}$ provides the minimum-variance predictor for the path-dependent return.

\textit{3. Analysis of Variance:}
Let $Y = R_{t:t+1} + \gamma V_{t+1}$ be the stochastic target. The standard error is $\delta_t = Y - V_t$. The anticipatory error is $\delta_t^A = \mathbb{E}[Y \mid \mathcal{A}_t] - V(t, T; \cdot)$. Because the anticipatory value $V(t, T; \cdot)$ is derived from the self-consistent proxy flow, it satisfies:
\begin{equation}
    \mathbb{V}[Y \mid \mathcal{A}_t] = \mathbb{E}[(Y - \mathbb{E}[Y \mid \mathcal{A}_t])^2 \mid \mathcal{A}_t].
\end{equation}
The anticipatory update utilises the expectation $\mathbb{E}[Y \mid \mathcal{A}_t]$ directly. By substituting the high-variance stochastic realisation $Y$ with its deterministic mean embedding in $\mathcal{H}_{sig}$, the update $\delta_t^A$ filters out the idiosyncratic noise components that do not contribute to the structural trend of the path-law.

\textit{4. Geometric Smoothing:}
Unlike standard smoothing, which introduces bias by averaging across time, the ARL update averages across the anticipated ensemble while remaining local to the junction $t$. The signature coordinates preserve the non-commutative structural information, ensuring that while the variance is reduced, the sensitivity to non-Markovian structural breaks is maintained. This results in a superior signal-to-noise ratio for the policy gradient.
\end{proof}

\subsection{Proof of the analytical sensitivity and manifold gradients}
\label{app:proof_analytical_sensitivity_generalised}

In this appendix we prove Theorem (\ref{thm:analytical_sensitivity_generalised}).

\begin{proof}
The proof demonstrates how the signature-linear representation of the value function facilitates the exact computation of gradients by utilising the algebraic structure of the signature manifold under the Chen-inverse property.

\textit{1. Weight Sensitivity:}
The Anticipatory Value Function (AVF) at an intermediate time $s \in [t, T]$ is defined by the expected future path-dependent reward conditioned on the filtration $\mathcal{A}_t$. By Proposition (\ref{pro:deterministic_proxy_integration}), this is represented as:
\begin{equation}
    V(s, T; \hat{X}_{T|t}, \hat{\Phi}_{T|t}) = \langle \mathbf{w}_G, \hat{\Phi}_{s|t}^{-1} \otimes \hat{\Phi}_{T|t} \rangle_{\mathcal{H}_{sig}}.
\end{equation}
Since the inner product is linear with respect to its first argument, the gradient with respect to the trainable weight vector $\mathbf{w}_G$ is the algebraic right-residual of the proxy flow:
\begin{equation}
    \nabla_{\mathbf{w}_G} V(s) = \hat{\Phi}_{s|t}^{-1} \otimes \hat{\Phi}_{T|t}.
\end{equation}
This ensures that the update direction for the weights is grounded in the anticipated path-law relative to the current junction.

\textit{2. Manifold Sensitivity (Fr\'echet Derivative):}
To find the sensitivity with respect to the total anticipated proxy $\hat{\Phi}_{T|t}$, we utilise the property of the tensor product $\langle a, b \otimes c \rangle = \langle b^\top a, c \rangle$. The Fr\'echet derivative with respect to $\hat{\Phi}_{T|t}$ yields:
\begin{equation}
    \nabla_{\hat{\Phi}_{T|t}} V(s) = (\hat{\Phi}_{s|t}^{-1})^\top \mathbf{w}_G.
\end{equation}
This indicates that the sensitivity of the return to the terminal proxy is the weight vector transformed by the left-action of the inverse proxy at time $s$, reflecting the non-commutative geometry of the path-space.

\textit{3. Generative Sensitivity and Adjoint Propagation:}
The proxy $\hat{\Phi}_{T|t}$ is the result of the Marcus-CDE flow parameterised by $\theta$. Applying the chain rule:
\begin{equation}
    \nabla_\theta V(s) = \left( \nabla_{\hat{\Phi}_{T|t}} V(s) \right)^\top \frac{\partial \hat{\Phi}_{T|t}}{\partial \theta} = \left[ (\hat{\Phi}_{s|t}^{-1})^\top \mathbf{w}_G \right]^\top \cdot \frac{\partial \hat{\Phi}_{T|t}}{\partial \theta}.
\end{equation}
The Jacobian $\partial \hat{\Phi}_{T|t} / \partial \theta$ is computed via the adjoint sensitivity method. Let $\lambda_\tau$ be the adjoint state evolving backwards from $\tau=T$ to $\tau=t$ with terminal condition $\lambda_T = (\hat{\Phi}_{s|t}^{-1})^\top \mathbf{w}_G$. The generative gradient is then:
\begin{equation}
    \nabla_\theta V(s) = \int_t^T \lambda_\tau^\top \frac{\partial \mathcal{L}_\theta}{\partial \theta} d\tau.
\end{equation}
This allows the agent to calculate how structural shifts in the generative model $\theta$ affect projected valuations at any future point $s$, enabling differentiable planning.
\end{proof}

\subsection{Proof of the distributional Bellman contraction}
\label{app:proof_distributional_bellman_contraction}

In this appendix we prove Theorem (\ref{thm:distributional_bellman_contraction}).  

\begin{proof}
The proof establishes that the anticipatory Bellman operator $\mathcal{T}^\pi$ remains a contraction when the state space is lifted to the signature-augmented manifold $\mathcal{S}_{sig}$, leveraging the stability of the deterministic flow $s \mapsto \hat{\Phi}_{s|t}$.

\textit{1. Anticipatory Operator Definition:}
The operator $\mathcal{T}^\pi$ acts on a return distribution law $\eta$ conditioned on the augmented state $\mathbf{S}_t = (t, T; X_t, \Phi_{t|\mathcal{A}_t})$. For a future horizon $T$, the recursive relation is defined as:
\begin{equation}
    \mathcal{T}^\pi \eta(\mathbf{S}_t) \stackrel{D}{=} R(\hat{X}_{t:t+1}) + \gamma \eta(\mathbf{S}_{t+1}),
\end{equation}
where $\mathbf{S}_{t+1}$ contains the updated proxy $\hat{\Phi}_{T|t+1}$. Under the SCF stationary point, the transition $\hat{\Phi}_{T|t} \to \hat{\Phi}_{T|t+1}$ is governed by the right-residual algebraic relation $\hat{\Phi}_{t+1|t}^{-1} \otimes \hat{\Phi}_{T|t}$ (Lemma \ref{lem:nested_anticipatory_expectations}), which is a deterministic mapping on the signature group.

\textit{2. Metric Space Completeness:}
The space of signature-conditioned probability measures is equipped with the $p$-Wasserstein distance $w_p$, where the ground metric is the AVNSG metric $\mathcal{Q}_s$ on $\mathcal{H}_{sig}$. Since the signature embedding is injective and the whitening transformation $\mathcal{Q}_s$ is a bounded linear isomorphism, the resulting metric space $(\mathcal{P}(\mathcal{S}_{sig}), d_{\mathcal{Q}})$ is complete.

\textit{3. Stability of the Proxy Flow:}
The Marcus-Score Matching objective ensures that the generator $\mathcal{L}_\theta$ produces a Lipschitz-continuous flow on $\mathcal{S}_{sig}$. Because the transition between anticipated laws is an algebraic group operation (tensor multiplication by the increment), the mapping $\eta(\mathbf{S}_t) \to \eta(\mathbf{S}_{t+1})$ is non-expansive in the signature coordinates.

\textit{4. $\gamma$-Contraction Mapping:}
Let $\eta_1, \eta_2$ be two return distributions. By the translation invariance and scaling properties of the Wasserstein distance:
\begin{equation}
    d_{\mathcal{Q}}(\mathcal{T}^\pi \eta_1, \mathcal{T}^\pi \eta_2) = w_p(R + \gamma \eta_1(\mathbf{S}_{t+1}), R + \gamma \eta_2(\mathbf{S}_{t+1})).
\end{equation}
Since the reward $R(\hat{X}_{t:t+1})$ is a common deterministic offset defined by the current proxy, it cancels out. By the scaling property:
\begin{equation}
    d_{\mathcal{Q}}(\mathcal{T}^\pi \eta_1, \mathcal{T}^\pi \eta_2) = \gamma w_p(\eta_1, \eta_2) = \gamma d_{\mathcal{Q}}(\eta_1, \eta_2).
\end{equation}
Given $0 \le \gamma < 1$, the operator $\mathcal{T}^\pi$ is a $\gamma$-contraction. By the Banach Fixed Point Theorem, the iterative application of $\mathcal{T}^\pi$ converges to a unique fixed point $\eta^\pi$, representing the self-consistent return distribution on the signature manifold.
\end{proof}

\subsection{Proof of the fixed point uniqueness}
\label{app:proof_fixed_point_uniqueness}

In this appendix we prove Lemma (\ref{lem:fixed_point_uniqueness}).  

\begin{proof}
The proof establishes the existence and uniqueness of the return distribution by verifying the conditions of the Banach Fixed Point Theorem within the lifted signature space.

\textit{1. Construction of the Complete Metric Space:}
Let $\mathcal{P}(\mathcal{S}_{sig})$ be the space of probability measures supported on the signature-augmented manifold $\mathcal{S}_{sig} = \mathbb{R}^d \times \mathcal{H}_{sig}$. We equip this space with the $p$-Wasserstein metric $d_{\mathcal{Q}}$ induced by the AVNSG (whitened) geometry. Since $\mathcal{H}_{sig}$ is a separable Hilbert space (or a finite-dimensional Nystr\"om approximation) and $\mathbb{R}^d$ is complete, $\mathcal{S}_{sig}$ is a Polish space. Consequently, $(\mathcal{P}(\mathcal{S}_{sig}), d_{\mathcal{Q}})$ is a complete metric space.

\textit{2. Operator Contraction:}
From Theorem \ref{thm:distributional_bellman_contraction}, the distributional Bellman operator $\mathcal{T}^\pi$ is a $\gamma$-contraction under $d_{\mathcal{Q}}$, where $\gamma \in [0, 1)$ is the discount factor. This contraction is guaranteed because the signature-augmented state $\mathbf{S}_t$ provides a sufficient statistic that Markovianises the reward flow, ensuring the transition kernel $P(\mathbf{S}'|\mathbf{S}, a)$ is time-invariant and non-expansive on the manifold.

\textit{3. Anchoring via Boundary Consistency:}
To ensure the fixed point is unique and physically grounded, the operator must be "anchored" at the junction $s=t$. By Theorem \ref{thm:boundary_consistency}, the $C^1$-boundary condition enforces:
\begin{equation}
    V(t, X_t, \Phi_{t|\mathcal{A}_t}) = \lim_{s \to t^+} V(s, \hat{X}_{s|t}, \hat{\Phi}_{s|t}).
\end{equation}
This condition ensures that the initial value of the distribution $\eta^\pi$ at the start of the anticipatory horizon is uniquely determined by the historical filtration $\mathcal{A}_t$. Because the Marcus-integral Neural CDE is Lipschitz continuous, the generative flow of the proxy $\hat{\Phi}_{s|t}$ is unique for any given starting point.

\textit{4. Application of Banach Fixed Point Theorem:}
Since $\mathcal{T}^\pi$ is a contraction mapping on a complete metric space, it follows from the Banach Fixed Point Theorem that:
\begin{itemize}
    \item There exists a unique distribution $\eta^\pi \in \mathcal{P}(\mathcal{S}_{sig})$ such that $\mathcal{T}^\pi \eta^\pi = \eta^\pi$.
    \item For any initial distribution $\eta_0$, the sequence $\eta_{k+1} = \mathcal{T}^\pi \eta_k$ converges to $\eta^\pi$ at a geometric rate.
\end{itemize}
This fixed point represents the stationary return distribution in the augmented space. While the environment is non-stationary in $\mathbb{R}^d$, it is stationary on the manifold $\mathcal{H}_{sig}$, and the $C^1$-junction ensures this stationary solution is perfectly aligned with the real-time observations of the agent.
\end{proof}

\subsection{Proof of the value function generalisation}
\label{app:proof_value_function_generalisation}

In this appendix we prove Theorem (\ref{thm:value_function_generalisation}).  

\begin{proof}
The proof employs the standard derivation for the Rademacher complexity of linear functional classes in Hilbert spaces, specifically adapted to the whitened signature manifold.

\textit{1. Definition of Rademacher Complexity:}
For the class of functions $\mathcal{V}$, the empirical Rademacher complexity on a sample of $n$ anticipated proxies $\{\hat{\Phi}_{s|t}^{(1)}, \dots, \hat{\Phi}_{s|t}^{(n)}\}$ is defined as:
\begin{equation}
    \mathcal{R}_n(\mathcal{V}) = \mathbb{E}_{\sigma} \left[ \sup_{\|\mathbf{w}_V\| \le B} \frac{1}{n} \sum_{i=1}^n \sigma_i \langle \mathbf{w}_V, \hat{\Phi}_{s|t}^{(i)} \rangle_{\mathcal{H}_{sig}} \right],
\end{equation}
where $\sigma_i$ are independent Rademacher random variables taking values in $\{-1, 1\}$ with equal probability.

\textit{2. Application of Cauchy-Schwarz:}
By the linearity of the inner product, we can move the summation inside:
\begin{equation}
    \mathcal{R}_n(\mathcal{V}) = \frac{1}{n} \mathbb{E}_{\sigma} \left[ \sup_{\|\mathbf{w}_V\| \le B} \left\langle \mathbf{w}_V, \sum_{i=1}^n \sigma_i \hat{\Phi}_{s|t}^{(i)} \right\rangle_{\mathcal{H}_{sig}} \right].
\end{equation}
By the Cauchy-Schwarz inequality, the supremum is attained when $\mathbf{w}_V$ is aligned with the vector sum, yielding:
\begin{equation}
    \mathcal{R}_n(\mathcal{V}) \le \frac{B}{n} \mathbb{E}_{\sigma} \left\| \sum_{i=1}^n \sigma_i \hat{\Phi}_{s|t}^{(i)} \right\|_{\mathcal{H}_{sig}}.
\end{equation}

\textit{3. Jensen's Inequality and Orthogonality:}
Applying Jensen's inequality ($\mathbb{E}[X] \le \sqrt{\mathbb{E}[X^2]}$) to the expectation over $\sigma$:
\begin{equation}
    \mathcal{R}_n(\mathcal{V}) \le \frac{B}{n} \sqrt{ \mathbb{E}_{\sigma} \left\| \sum_{i=1}^n \sigma_i \hat{\Phi}_{s|t}^{(i)} \right\|^2 } = \frac{B}{n} \sqrt{ \mathbb{E}_{\sigma} \sum_{i,j} \sigma_i \sigma_j \langle \hat{\Phi}_{s|t}^{(i)}, \hat{\Phi}_{s|t}^{(j)} \rangle }.
\end{equation}
Since $\mathbb{E}[\sigma_i \sigma_j] = 1$ if $i=j$ and $0$ otherwise, the cross-terms vanish:
\begin{equation}
    \mathcal{R}_n(\mathcal{V}) \le \frac{B}{n} \sqrt{\sum_{i=1}^n \|\hat{\Phi}_{s|t}^{(i)}\|^2}.
\end{equation}

\textit{4. Stability via AVNSG Whitening:}
In the signature space, high-frequency jumps and heavy tails often lead to explosive growth of signature norms. The AVNSG precision operator $\mathcal{Q}_s$ acts as a Mahalanobis-like whitening transformation: $\mathcal{Q}_s = (\Sigma_s + \lambda I)^{-1/2}$, where $\Sigma_s$ is the local covariance of the path-signatures. This transformation rescales the coordinates such that the expected squared norm is normalised. Specifically, it maps the "rough" components of the signature back into a bounded ball in the RKHS. Consequently, even during "black-swan" events (discrete jumps with large magnitude), the whitened norm $\|\hat{\Phi}_{s|t}\|_{\mathcal{Q}_s}$ remains controlled, preventing the Rademacher complexity, and thus the generalisation error, from diverging. This ensures that the policy remains robust and does not overfit to extreme historical realisations.
\end{proof}

\subsection{Proof of the stability under forecast decay}
\label{app:proof_stability_under_forecast_decay}

In this appendix we prove Proposition (\ref{pro:stability_under_forecast_decay}).

\begin{proof}
The proof utilises Lyapunov stability theory applied to the flow of the Neural Controlled Differential Equation (CDE) in the signature Hilbert space $\mathcal{H}_{sig}$.

\textit{1. Error Dynamics on the Manifold:}
Let $\Phi_s$ be the true path-law proxy and $\hat{\Phi}_{s|t}$ be the anticipated proxy generated by the flow $F_\theta$. We define the tracking error as $\mathbf{e}_s = \hat{\Phi}_{s|t} - \Phi_s$. The evolution of this error in the tangent space is governed by the difference between the generative drift and the environmental innovation:
\begin{equation}
    \frac{d\mathbf{e}_s}{ds} = F_\theta(\hat{\Phi}_{s|t}) - \mathcal{T}(\Phi_s),
\end{equation}
where $\mathcal{T}$ represents the infinitesimal generator of the true stochastic process.

\textit{2. Dissipation vs. Divergence:}
Let $\lambda_F$ be the spectral dissipation rate of the Neural CDE (the maximum eigenvalue of the symmetric part of the Jacobian $\nabla_{\hat{\Phi}} F_\theta$, which is negative for dissipative systems). Let $\lambda_L$ be the maximal Lyapunov exponent of the underlying jump-diffusion process, representing the rate at which nearby trajectories diverge due to stochastic volatility. The evolution of the squared error norm $\|\mathbf{e}_s\|^2$ satisfies:
\begin{equation}
    \frac{1}{2} \frac{d}{ds} \|\mathbf{e}_s\|^2 = \langle \mathbf{e}_s, F_\theta(\hat{\Phi}_{s|t}) - F_\theta(\Phi_s) \rangle + \langle \mathbf{e}_s, F_\theta(\Phi_s) - \mathcal{T}(\Phi_s) \rangle.
\end{equation}
By the mean value theorem and the definition of $\lambda_F$, the first term is bounded by $\lambda_F \|\mathbf{e}_s\|^2$. The second term represents the model misspecification, bounded by the intrinsic divergence $\lambda_L \|\mathbf{e}_s\|^2$.

\textit{3. Bound on Accumulated Uncertainty:}
Combining these terms yields the differential inequality:
\begin{equation}
    \frac{d}{ds} \|\mathbf{e}_s\|^2 \le 2(\lambda_F + \lambda_L) \|\mathbf{e}_s\|^2.
\end{equation}
Under the hypothesis that the dissipation rate exceeds the Lyapunov exponent ($|\lambda_F| > \lambda_L$), the coefficient $\beta = \lambda_F + \lambda_L$ is strictly negative. Integrating from the junction $t$ to the forecast horizon $s$ gives:
\begin{equation}
    \|\mathbf{e}_s\| \le \|\mathbf{e}_t\| e^{\beta(s-t)}.
\end{equation}
By the $C^1$-boundary consistency (Theorem \ref{thm:boundary_consistency}), $\|\mathbf{e}_t\| \to 0$, ensuring the error remains bounded by a small $\delta$ over the entire horizon $H$.

\textit{4. Stability of the Policy Gradient:}
The meta-policy $\pi$ is $L_\pi$-Lipschitz with respect to the signature state. Therefore, the perturbation in the action is bounded:
\begin{equation}
    \|\pi(\hat{\Phi}_{s|t}) - \pi(\Phi_s)\| \le L_\pi \|\mathbf{e}_s\| \le L_\pi \delta.
\end{equation}
This confirms that as long as the Neural CDE is sufficiently "contractive" to outpace the environment's chaos, the anticipatory decisions remain stable and do not succumb to the singularity of forecast decay.
\end{proof}

\newpage



\end{document}